\documentclass[11pt]{article}
\usepackage[utf8]{inputenc}
\usepackage[margin=1in, letterpaper]{geometry}
\usepackage[T1]{fontenc}    % use 8-bit T1 fonts
\usepackage{hyperref}       % hyperlinks
\usepackage{url}            % simple URL typesetting
\usepackage{booktabs}       % professional-quality tables
\usepackage{amsfonts}       % blackboard math symbols
\usepackage{nicefrac}       % compact symbols for 1/2, etc.
\usepackage{microtype}
\usepackage{algorithm}
\usepackage[noend]{algpseudocode}
\usepackage{amsmath,amssymb,amsthm,float,tkz-berge,caption,mathtools}
\usepackage{enumerate}
\usepackage{subfig}
\usepackage{booktabs}
\usepackage{bold-extra}
\usepackage{caption}
\usepackage{wrapfig}
\usepackage{setspace}

\newtheorem{theorem}{Theorem}
\newtheorem*{theorem*}{Theorem}
\newtheorem{lemma}{Lemma}

\newtheorem*{definition*}{Definition}
\newtheorem*{lemma*}{Lemma}

\newtheorem*{corollary*}{Corollary}

\newtheorem*{claim*}{Claim}

\newtheorem*{proposition*}{Proposition}
\newtheorem{assumption}{Assumption}
\newtheorem*{assumption*}{Assumption}
\allowdisplaybreaks
\title{Robust Estimation of Tree Structured Ising Models}
\author{Ashish Katiyar \\ \href{mailto:a.katiyar@utexas.edu}{a.katiyar@utexas.edu} 
\and Vatsal Shah\\ \href{mailto:vatsalshah1106@utexas.edu}{vatsalshah1106@utexas.edu}
\and Constantine Caramanis\\ \href{mailto:constantine@utexas.edu}{constantine@utexas.edu}}
\date{The University of Texas at Austin}

\begin{document}
	\maketitle
% \doublespacing
\begin{abstract}
       We consider the task of learning Ising models when the signs of different random variables are flipped independently with possibly unequal, unknown probabilities. In this paper, we focus on the problem of robust estimation of tree-structured Ising models. Without any additional assumption of side information, this is an open problem. We first prove that this problem is unidentifiable, however, this unidentifiability is limited to a small equivalence class of trees formed by leaf nodes exchanging positions with their neighbors.  Next, we propose an algorithm to solve the above problem with logarithmic sample complexity in the number of nodes and polynomial run-time complexity. Lastly, we empirically demonstrate that, as expected, existing algorithms are not inherently robust in the proposed setting whereas our algorithm correctly recovers the underlying equivalence class.
\end{abstract}

\maketitle
\newcommand{\ind}[4]{{#1}_{#2}^{#3, #4}}
\newcommand{\bs}[1]{\boldsymbol{#1}}
\newcommand{\E}[1]{\mathbb{E}{\left[#1\right]}}
\newcommand{\tempa}[2]{\mathbb{E}\left[{#1}_{#2}\right]}
\newcommand{\T}{\mathcal{T}_{T^*}}
\newcommand{\fa}{\textsc{FindEC}}
\newcommand{\fb}{\textsc{SplitTree}}
\newcommand{\fc}{\textsc{Recurse}}
\newcommand{\setx}{\mathcal{X}}
\newcommand{\sete}{\mathcal{E}}
\newcommand{\setp}[1]{\mathcal{P}_{i}^{#1}}
\newcommand{\Or}{\mathcal{O}}
\newcommand{\es}{\mathcal{E}^*}
\newcommand{\ep}{\mathcal{E}'}
\newcommand{\ee}{\mathcal{E}^e}
\newcommand{\corr}[2]{\rho_{#1, #2}}
\newcommand{\cov}[2]{\Sigma_{#1, #2}}

\section{Introduction}

Undirected graphical models or Markov Random Fields (MRFs) are a powerful tool for describing high dimensional distributions using an associated dependency graph $\mathbb{G}$, which encodes the conditional dependencies between random variables. In this paper, we focus on a special class of MRFs called the Ising Model, first introduced in \cite{ising1925beitrag} to represent spin systems in quantum physics \cite{brush1967history}. Recently, Ising models have also proven quite popular in biology \cite{jaimovich2006towards}, engineering \cite{choi2010exploiting, schneidman2006weak}, computer vision \cite{roth2005fields}, and also in the optimization and OR communities, including in finance \cite{zhou2007self}, and social networks \cite{mesnards2018detecting}. The special class of tree-structured Ising models is beneficial for applications in statistical physics over non-amenable graphs. A detailed description and further references can be found in \cite{martinelli2003ising}.

In this paper, we explore the problem of learning the underlying graph of tree-structured Ising models with independent, unknown, unequal error probabilities. Such scenarios are relatively common in networks where certain malfunctioning sensors can report flipped bits. The presence of noise breaks down the conditional independence structure of the noiseless setting, and the noisy samples may no longer follow an Ising model distribution.

In 2011, \cite{chenlearning} highlighted the importance of robustness in Ising models. 
Recent works in \cite{goel2019learning,hamilton2017information,lindgrenrobust} have tried to address this problem. However, they assume the side information of the error probability, which is mostly unavailable and difficult to estimate in most practical settings. In the closely related work for tree-structured Ising models, \cite{pmlr-v89-nikolakakis19a,nikolakakis2019non} address this problem as they build on the seminal work of \cite{chow1968approximating}, popularly known as the Chow-Liu algorithm, which proposes the use of maximum-weight spanning tree of mutual information as an estimate of the tree structure. In \cite{pmlr-v89-nikolakakis19a}, they consider the simplified case where each node has an equal probability of error and \cite{nikolakakis2019non} assumes that the error doesn't alter the order of mutual information. Both assumptions imply that asymptotically, Chow-Liu converges to the correct tree. However, these assumptions don't arise naturally and are difficult to check from access to only noisy data. To the best of our knowledge, there doesn't exist an analysis of what happens beyond this limiting assumption of order preservation of mutual information.

In fact, section 5.1 of \cite{bresler2008reconstruction} provides an example of the unidentifiability of the problem for a graph on 3 nodes and says that the problem is ill-defined. We reconsider this problem, and show that for the special class of tree structured Ising models, although the problem is not identifiable, nevertheless the unidentifiability is limited to an equivalence class of trees. Thus, more appropriately, one can cast the problem of learning in the presence of unknown, unequal noise as the problem of learning this equivalence class.

\paragraph{Key Contributions}
\begin{enumerate}
    \item We show that the problem of learning tree structured Ising models when the observations flip with independent, unknown, possibly unequal probability is unidentifiable (Theorem \ref{thm:unident}).
    \item As depicted in Figure \ref{fig:identifiability}, the unidentifiability is restricted to the equivalence class of trees obtained by permuting within the leaf nodes and their neighbors (Theorem \ref{thm:lim_unident}).
    \item We provide an algorithm to recover this equivalence class from noisy samples. The sample complexity is logarithmic and the time complexity is polynomial in the number of nodes. (Section \ref{sec:algo})
    \item We experimentally demonstrate that existing algorithms like Chow-Liu or Sparsitron for learning Ising models are not inherently robust against this noise model whereas our algorithm correctly recovers the equivalence class.(Section \ref{sec:exp})
\end{enumerate}
\section{Related Work}
Efficient algorithms for structure learning of Ising models can be divided into three main categories based on their assumptions: i) special graph structures \cite{anandkumar2012learning,chow1968approximating,dasgupta1999learning,srebro2003maximum, bresler2016learning}, ii) nature of interaction between variables such as correlation decay property (CDP) \cite{bresler2014hardness, bresler2014structure, bresler2008reconstruction, lee2007efficient, ravikumar2010high}, iii) bounded degree/width \cite{bresler2015efficiently, daskalakis2019testing, klivans2017learning, lokhov2018optimal, wu2018sparse, vuffray2016interaction}. However, these algorithms assume access to uncorrupted samples.

In the last decade, there has been a lot of research on robust estimation of graphical models \cite{pmlr-v97-katiyar19a, kolar2012estimating,liu2012high,loh2011high,wang2017robust,yang2015robust}, and \cite{pmlr-v97-katiyar19a} explicitly considered partial identifiability in the presence of noise for the Gaussian case. However, extending the above frameworks to the robust structure learning of Ising models remains a challenge. 
\cite{goel2019learning,hamilton2017information,lindgrenrobust} have tried to solve the problem of robust estimation of general Ising models under the assumption of access to the probability of error for each node. Recently, \cite{pmlr-v89-nikolakakis19a, nikolakakis2019non} proposed algorithms to estimate the underlying graph structure of tree-structured Ising models in the presence of noise under the strong assumption that the probability of error does not alter the order of mutual information order for the tree. Both these assumptions are restrictive and impractical. In this paper, we present the first algorithm that can robustly recover the underlying tree structured Ising model (upto an equivalence class) in the presence of corruption via unknown, unequal, independent noise. 

\section{Identifiability Result}
\paragraph{Problem Setup:}
Let $\setx = \{X_i^*: i\in[n]\}$ be a set of random variables with support on $\{-1, 1\}$ and $\mathbf{X} = [X_1^*, X_2^*\dots X_n^*]$ be the vector of these random variables. Suppose the conditional independence structure of the variables of $\setx$ is given by a tree $T^*$. This implies that the distribution of $\setx$ can be represented by an Ising model. In our model, we have observations where each $X_i^*$ flips with probability $q_i$. We denote the probability of error by the vector $\mathbf{q} = [q_1, q_2, \dots q_n]$ and the noisy random variables by $\setx^e = \{X_i^e: i\in[n]\}$. The error in $X_i^*$ disrupts the tree structured conditional independence and the graphical model of $\setx^e$ is a complete graph if $q_i > 0$ $\forall i\in[n]$. In fact, $\setx^e$ need not be an Ising model. Given samples of $\setx^e$, we want to find the tree structure $T^*$. The ideas of the analysis for this problem are based on graph structure properties introduced in \cite{pmlr-v97-katiyar19a} where they were applied in the different context of Gaussian graphical models.
 \subsubsection*{Equivalence Class for $T^*$} 
% We include the definition of $\T$ from \cite{pmlr-v97-katiyar19a} for completeness.
Given a tree $T^*$, let us define its equivalence class denoted by $\T$. 
Let the set of all nodes that are connected to one or more leaf nodes of $T^*$ be denoted by $\mathcal{A}$.
% \begin{equation*}
%     \mathcal{P} = \{a \mid \text{node } a \text{ is connected to a leaf node in } T^*\}.
% \end{equation*}
Construct sets $\mathcal{L}_{a_1}, \mathcal{L}_{a_2}, \dots, \mathcal{L}_{a_{|\mathcal{A}|}}$ from the nodes $a_i \in \mathcal{A}$ such that set $\mathcal{L}_{a_i}$ contains all the leaf nodes connected to $a_i$ in $T^*$ along with $a_i$. Sample one node from each newly constructed set $\mathcal{L}_{(.)}$ and create another subset $\mathcal{S}^m$. There exists $q = |\mathcal{L}_{a_1}| |\mathcal{L}_{a_2}|\dots|\mathcal{L}_{a_{|\mathcal{A}|}}|$ such unique subsets.
% Create a tree $T^m$ by exchanging each node in $\mathcal{S}^m$ with its neighbor or itself in the tree $T^*$.
Create a tree $T^m$ by setting all the nodes in $\mathcal{S}^m$ as the internal node and all the remaining nodes in their corresponding sets $\mathcal{L}_{(.)}$ as leaf nodes in the same position where $\mathcal{L}_{(.)}$ was in $T^*$.
Therefore, there exists a one-to-one equivalence relationship between tree $T^m$ and its corresponding set $\mathcal{S}^m$. We define a set of these trees as $\mathcal{T}_{T^*}$. 
%Consider a set of trees obtained by exchanging any subset of leaf nodes with their parent node in $T^*$. Let the set   
\begin{equation*}
\mathcal{T}_{T^*} = \{ T^m \mid m \in \{1, 2, \hdots q\}  \}.
\end{equation*}
\begin{figure}
    \centering
    \includegraphics[scale = 0.4]{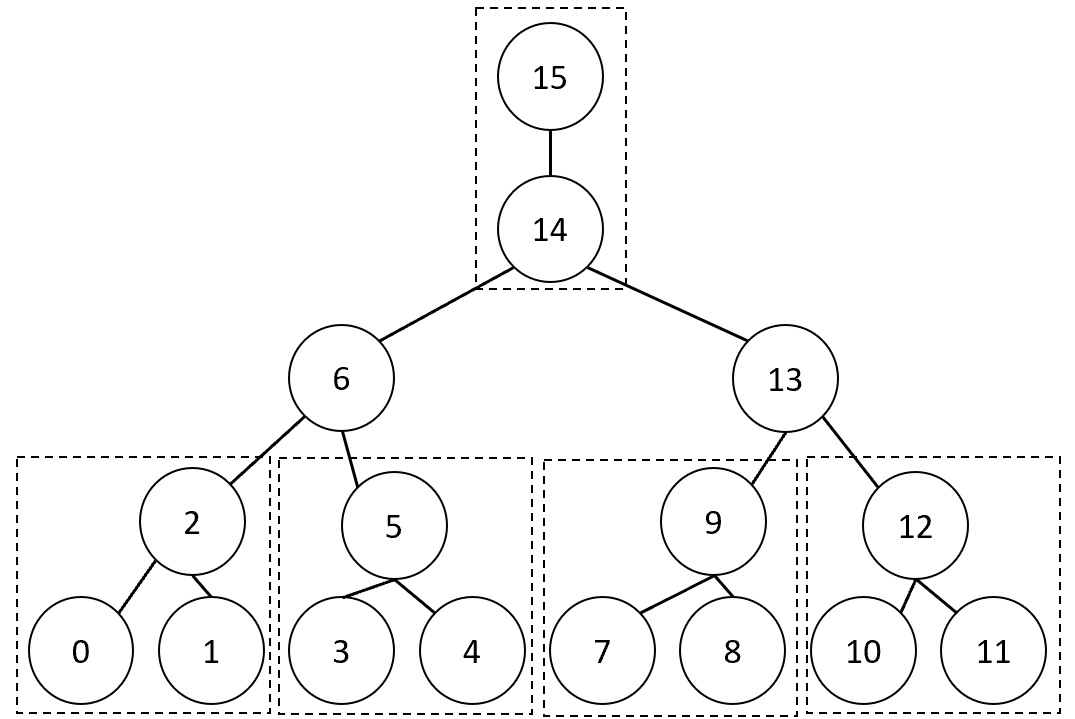}
    \caption{$\mathcal{T}_{T^*}$ for the above $T^*$ is obtained by permuting the nodes within the dotted regions. Noisy samples from this Ising model can be used to estimate the tree only upto the equivalence class $\T$.}
    \label{fig:identifiability}
\end{figure}
Figure \ref{fig:identifiability} gives an example of $\mathcal{T}_{T^*}$.
\paragraph{Model Assumptions}
\begin{assumption}\label{as:bounded_mean}
(Bounded Mean) The absolute value of the mean - $|\E{X_i}|\leq \mu_{max}<1$ $\forall i\in[n]$.
\end{assumption}
\begin{assumption}\label{as:bounded_corr}
(Bounded Correlation) Correlation $\rho_{i,j}$ of any two nodes $X_i$ and $X_j$ connected by an edge - $\rho_{min}\leq|\rho_{i,j}|\leq\rho_{max}$ where $0<\rho_{min}\leq\rho_{max}<1$.
\end{assumption}
\begin{assumption}\label{as:p_e_bound}
(Bounded error probability) The error probability - $0\leq q_i \leq q_{max} < 0.5$ $\forall i\in[n]$.
\end{assumption} 
These assumptions arise naturally. \textit{Assumption \ref{as:bounded_mean}} ensures that no variable approaches a constant and hence gets disconnected from the tree.
The lower bound in \textit{Assumption \ref{as:bounded_corr}} also ensures that every node is connected. The upper bound in \textit{Assumption \ref{as:bounded_corr}} ensures that no two nodes are duplicated. \textit{Assumption \ref{as:p_e_bound}} ensures the noisy node doesn't become independent of every other node due to the error.
%Note that assumption \ref{as:symmetry} implies symmetry for any subset of the nodes both before and after introducing errors in the random variables. 
% \textit{Assumption \ref{as:bounded_mean}} ensures that no variable approaches a constant and has a minimum variance of $(1-\mu_{max}^2)$. \textit{Assumption \ref{as:bounded_corr}} is only for nodes connected by an edge. By the correlation decay property, this implies that the correlation between any arbitrary pair of nodes is also upper bounded by $\rho_{max}$. The lower bound on the correlation between an arbitrary pair of nodes is allowed to scale with the number of nodes. 
% % Also note that there is no assumption of positive or negative correlation between nodes attached by an edge (no ferromagnetism or anti-ferromagnetism assumption).
% \textit{Assumption \ref{as:p_e_bound}} is a natural assumption to make. When $q_i = 0.5$, we get observations from an independent random variable and therefore it is impossible to find its edges. Also, the analysis in this paper is for $q_i \leq q_{max} < 0.5$ for notational simplicity, the whole analysis works as long as $\min\{q_i, 1-q_i\} \leq q_{max} < 0.5$.
\subsection*{Limited unidentifiability of the problem}
In Theorem \ref{thm:lim_unident}, we prove that it is possible to recover $\T$ from the samples of $\setx^e$. Further, we prove that given the distribution of $\setx^e$, there exists an Ising model for each tree in $\T$ such that, for some noise vector, its noisy distribution is the same as that of $\setx^e$ in Theorem \ref{thm:unident}
% We first prove in Theorem \ref{thm:lim_unident} that in the infinite sample limit, it is possible to recover $\T$ from the samples of $\setx^e$. From there, we go on to prove in Theorem \ref{thm:unident} that we cannot do any better than estimating $\T$ as, given the distribution of $\setx^e$, there exists an Ising model for each tree in $\T$ such that, for some noise vector, its noisy distribution is the same as that of $\setx^e$.
\begin{theorem}\label{thm:lim_unident}
%Let $X^*$ represents a set of binary random variables satisfying Assumptions \ref{as:symmetry}, \ref{as:positive_corr} with conditional independence structure given by $T^*$. Suppose $X'$ indicates any set of binary random variables whose conditional independence relation is given by a tree $T'$
Suppose $\setx'$ and $\setx$ are sets of binary valued random variables satisfying assumption \ref{as:bounded_mean} whose conditional independence is given by trees $T'$ and $T^*$ respectively satisfying assumption \ref{as:bounded_corr}. Assume that each node in both these distributions $T'$ and $T^*$ is allowed to be flipped independently with probability satisfying assumption 3. Let $\es$ and $\ep$ represent the noisy distributions of $\setx$ and $\setx'$ respectively. If $\ep=\es$, then $T' \in \mathcal{T}_{T^*}$.
\end{theorem}

% \begin{theorem}
% Suppose $X'$ and $X^*$ are sets of binary valued random variables whose conditional independence is given by trees $T'$ and $T^*$ respectively and satisfy assumptions \ref{as:symmetry} and \ref{as:positive_corr}. Now, the nodes of $X'$ and $X^*$ flip independently with probability  $\mathbf{q'}<0.5$ and $\mathbf{q}<0.5$ respectively such that the resultant noisy probability distribution of both the noisy distributions is equal. Then $T' \in$ \T.
% \end{theorem}
\begin{proof}
The proof of this theorem relies on this key observation: the probability distribution of the noisy samples completely defines the categorization of any set of 4 nodes as star/non-star shape. Any set of 4 nodes forms a non-star if there exists at least one edge which, when removed, splits the tree into two subtrees such that exactly 2 of the 4 nodes lie in one subtree and the other 2 nodes lie in the other subtree. The nodes in the same subtree form a pair. If the set is not a non-star, it is categorized as a star. For example, in Figure \ref{fig:identifiability}, $\{0,1,7,12\}$ form a non-star and $\{0,7,12,13\}$ form a star.

Once we prove this key observation, the rest of the proof follows easily and we refer the reader to parts (ii) and (iii) of the proof of Theorem 2 from \cite{pmlr-v97-katiyar19a}.
We denote the correlation between two nodes $X_i$ and $X_j$ in the non-noisy setting by $\rho_{i,j}$ and in the noisy setting by $\rho'_{i,j}$. Similarly the covariance is denoted by $\Sigma_{i,j}$ and $\Sigma'_{i,j}$. We utilize the correlation decay property of tree structured Ising models which is stated in Lemma \ref{lemma:corr_decay}.
\begin{lemma}\label{lemma:corr_decay}
% (Correlation Decay) The correlation for any 2 nodes $X_{i_1}$ and $X_{i_k}$ in a tree structured Ising Model such that the path between them is  $(X_{i_1}\rightarrow X_{i_2} \rightarrow X_{i_3}\dots  \rightarrow X_{i_k})$ is:
% \begin{equation}\label{eq:corr_decay}
%     \rho_{i_1i_k} = \prod_{l=2}^{k}\rho_{i_{l-1},i_l}.
% \end{equation}
(Correlation Decay) Any 2 nodes $X_{i_1}$ and $X_{i_k}$ have the conditional independence relation specified by a tree structured Ising Model such that the path between them is  $(X_{i_1}\rightarrow X_{i_2} \rightarrow X_{i_3}\dots  \rightarrow X_{i_k})$ if and only if their correlation is given by:
\begin{equation}\label{eq:corr_decay}
    \rho_{i_1i_k} = \prod_{l=2}^{k}\rho_{i_{l-1},i_l}.
\end{equation}
\end{lemma}
The proof of this lemma is provided in Appendix, \ref{app:corr_decay}.
We also prove that $\E{X_i^e} = (1-2q_i)\E{X_i}$ and $\Sigma'_{i,j} = (1-2q_i)(1-2q_j)\Sigma_{i,j}$ in Appendix \ref{app:cov_noisy}.
% We know that $\rho_{i,j} = \frac{\Sigma_{i,j}}{\sqrt{\Sigma_{i,i}\Sigma_{j,j}}}$ and $\rho'_{i,j} = \frac{\Sigma'_{i,j}}{\sqrt{\Sigma'_{i,i}\Sigma'_{j,j}}}$.

% We also need the following relation between $\Sigma_{i,j}$ and $\Sigma'_{i,j}$,
% \begin{equation}\label{eq:noisy_cov}
% \Sigma'_{i,j} = (1-2q_i)(1-2q_j)\Sigma_{i,j}.
% \end{equation}
% We prove this in the appendix. An intermediate step in the proof involves proving:

% \begin{equation}\label{eq:noisy_exp}
% \E{X_i'} = (1-2q_i)\E{X_i}.
% \end{equation}
\subsection*{Categorizing a set of 4 nodes as star/non-star}

We first look at a graphical model on 3 nodes $X_1, X_2, X_3$ whose conditional independence is given by a chain with $X_2\perp X_3|X_1$.  By Lemma \ref{lemma:corr_decay}, we have $\Sigma_{2,3}\Sigma_{1,1} = \Sigma_{1,2}\Sigma_{1,3}$.
% \begin{equation}\label{eq:corr_decay_3node}
% \Sigma_{2,3}\Sigma_{1,1} = \Sigma_{1,2}\Sigma_{1,3}
% \end{equation}

Suppose the sign of $X_1, X_2, X_3$ flip independently with probability $q_1, q_2, q_3$ respectively. Substituting the values of $\Sigma_{2,3}, \Sigma_{1,1}, \Sigma_{1,2}$ and $\Sigma_{1,3}$ in terms of their noisy counterparts gives us:
% From equation (\ref{eq:noisy_cov}), we have:
% \begin{equation}\label{eq:noisy_cov_3node}
% \begin{aligned}
% \Sigma'_{1,2} &= (1-2q_1)(1-2q_2)\Sigma_{1,2}\\
% \Sigma'_{1,3} &= (1-2q_1)(1-2q_3)\Sigma_{1,3}\\
% \Sigma'_{2,3} &= (1-2q_2)(1-2q_3)\Sigma_{2,3}.
% \end{aligned}
% \end{equation}

% Utilizing Equation \ref{eq:noisy_exp} and the fact that $\E{X_i'^2} = 1, \E{X_i^2} = 1$, we get:
% \begin{equation}
% \begin{aligned}
% \Sigma'_{1,1} &= 1 - (\E{X_1} (1 - 2q_1))^2,\\
% \Sigma_{1,1} &= 1 - \E{X_1}^2.
% \end{aligned}
% \end{equation}
% This results in:
% \begin{equation} \label{eq:noisy_var_3node}
% \Sigma_{1,1} = 1 - \frac{1-\Sigma'_{1,1}}{(1 - 2q_1)^2}
% \end{equation}
% Substituting the values from Equations (\ref{eq:noisy_cov_3node}) and (\ref{eq:noisy_var_3node}) in Equation (\ref{eq:corr_decay_3node}), we get:
\begin{equation}\label{eq:quadratic}
(1-2q_1)^2 = 1 - \Sigma'_{1,1} + \frac{\Sigma'_{1,2}\Sigma'_{1,3}}{\Sigma'_{2,3}}.
\end{equation}

If we had prior knowledge about the underlying conditional independence relation, this quadratic equation, which depends only on the quantities measurable from noisy data, could be solved to estimate the probability of error of $X_1$.

We prove in Appendix \ref{sec:valid_sol} that Equation (\ref{eq:quadratic}) gives a valid solution for any configuration of 3 nodes in a tree structured Ising model. Therefore, in the absence of the knowledge that $X_2\perp X_3|X_1$, we can estimate a probability of error for each $X_i$ which enforces the underlying graph structure to represent the other 2 nodes independent conditioned on $X_i$. Thus, irrespective of the true underlying conditional independence relation we can always find a probability of error for each node which makes any other pair of nodes conditionally independent. We use this concept to classify a tree on 4 nodes as star or non-star shaped.

We follow a notation where $\hat{q}_i^{j,k}$ denotes the estimated probability of error of $X_i$ which enforces $X_j\perp X_k|X_i$. \vspace{-3pt}
% Further, we denote the $b$ and $c$ in equation \ref{eq:quadratic} corresponding to $\hat{q}_i^{j,k}$ by $\ind{b}{i}{j}{k}$ and $\ind{c}{i}{j}{k}$ respectively.
 \subsubsection*{Condition for star/non-star shape:}\vspace{-3pt}
Any set of 4 nodes $\{X_1, X_2, X_3, X_4\}$ is categorized as a non-star with $(X_1, X_2)$ forming one pair and $(X_3, X_4)$ forming another pair if and only if:
\begin{equation*}\label{eq:non_star_cond}
\begin{aligned}
    \ind{\hat{q}}{1}{2}{3} = \ind{\hat{q}}{1}{2}{4} \neq \ind{\hat{q}}{1}{3}{4},
    \ind{\hat{q}}{2}{1}{3} = \ind{\hat{q}}{2}{1}{4} \neq \ind{\hat{q}}{2}{3}{4},\\
    \ind{\hat{q}}{3}{2}{4} = \ind{\hat{q}}{3}{1}{4} \neq \ind{\hat{q}}{3}{1}{2},
    \ind{\hat{q}}{4}{2}{3} = \ind{\hat{q}}{4}{1}{3} \neq \ind{\hat{q}}{4}{1}{2}.
\end{aligned}
\end{equation*}
From Equation ($\ref{eq:quadratic}$), this is equivalent to the condition that $\frac{\corr{1}{3}'}{\corr{1}{4}'} = \frac{\corr{2}{3}'}{\corr{2}{4}'}, \frac{\corr{1}{2}'}{\corr{1}{4}'}\neq \frac{\corr{3}{2}'}{\corr{3}{4}'}.$

Any set of 4 nodes $\{X_1, X_2, X_3, X_4\}$ is categorized as a star if and only if:
\begin{equation*}\label{eq:star_cond}
\begin{aligned}
    \ind{\hat{q}}{1}{2}{3} = \ind{\hat{q}}{1}{2}{4} = \ind{\hat{q}}{1}{3}{4},
    \ind{\hat{q}}{2}{1}{3} = \ind{\hat{q}}{2}{1}{4} = \ind{\hat{q}}{2}{3}{4},\\
    \ind{\hat{q}}{3}{2}{4} = \ind{\hat{q}}{3}{1}{4} = \ind{\hat{q}}{3}{1}{2},
    \ind{\hat{q}}{4}{2}{3} = \ind{\hat{q}}{4}{1}{3} = \ind{\hat{q}}{4}{1}{2}.
    \end{aligned}
\end{equation*}
This is equivalent to the condition that $\frac{\corr{1}{3}'}{\corr{1}{4}'} = \frac{\corr{2}{3}'}{\corr{2}{4}'} =  \frac{\corr{1}{2}'}{\corr{1}{4}'}$. 

In order to see how these conditions correspond to a star/non-star shape, lets consider a chain on 4 nodes as shown in Figure \ref{fig:chain}. 
Let each $X_i$ be flipped with probability $q_i$. With access only to the noisy samples, we estimate the probability of error for each node in order to find the underlying tree. The key idea is that when we estimate the probability of error for a given node, it should be consistent across different conditional independence relations. For instance in the present case, the error estimates $\hat{q}_2^{1,3}$ and $\hat{q}_2^{1,4}$ of $X_2$ satisfy $\hat{q}_2^{1,3} = \hat{q}_2^{1,4} = q_2$. We show that $\hat{q}_2^{3,4} \neq \hat{q}_2^{1,3}$ (Lemma \ref{lemma:equality_inequality}(b)). We also prove that $\hat{q}_1^{2,3} = \hat{q}_1^{2,4} \neq \hat{q}_1^{3,4}$(Lemma \ref{lemma:equality_inequality}). These imply that $X_3\not\perp X_4|X_2$ and $X_3\not\perp X_4|X_1$. By symmetry, we have $X_1\not\perp X_2|X_3$ and $X_1\not\perp X_2|X_4$. These conditional independence statements imply that $X_1, X_2, X_3$ and $X_4$ form a chain with $(X_1, X_2)$ on one side of the chain and $(X_3, X_4)$ on the other side of the chain. 

Next, we consider the case when 4 nodes form a star structured graphical model as in Figure \ref{fig:star}. Under the same noisy observation setting we prove that $\hat{q}_1^{2,3} = \hat{q}_1^{2,4} = \hat{q}_1^{3,4}$, $\hat{q}_2^{1,3} = \hat{q}_2^{1,4} = \hat{q}_2^{3,4}$, $\hat{q}_3^{1,2} = \hat{q}_3^{1,4} = \hat{q}_3^{2,4}$ and $\hat{q}_4^{1,3} = \hat{q}_4^{1,2} = \hat{q}_4^{3,2}$ (Lemma \ref{lemma:star_equality}). Thus, we can conclude that the underlying graphical model is star structured.
% \twocolumn
\begin{figure}
% \centering
    \begin{minipage}[t]{.5\linewidth}
    \centering
    \includegraphics[width=.5\linewidth]{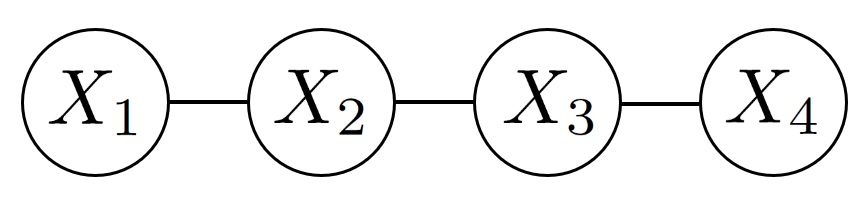}
    \caption{A chain structure.}
    \label{fig:chain}
    \end{minipage}
    \hfill
    \begin{minipage}[t]{.5\linewidth}
    \centering
    \includegraphics[width=.35\linewidth]{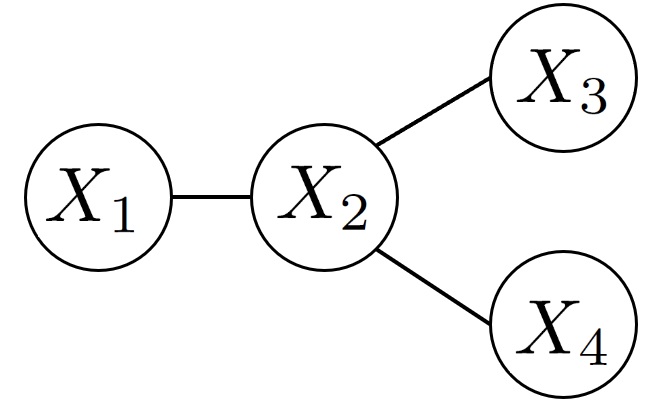}
    \caption{A Star structure.}
    \label{fig:star}
    \end{minipage}
\end{figure}
% \onecolumn
\begin{lemma}\label{lemma:equality_inequality}
Let the graphical model on $X_1, X_2, X_3$ and $X_4$ form a chain as shown in Figure $\ref{fig:chain}$. Suppose the bits of each $X_i$ are flipped with probability $q_i<0.5$. Then the following holds:
\begin{equation*}
    \text{(a) } \hat{q}_1^{2,3} = \hat{q}_1^{2,4},\text{ (b) } \hat{q}_2^{1,3} \neq \hat{q}_2^{3,4},\ind{\hat{q}}{1}{2}{3} \neq \ind{\hat{q}}{1}{3}{4}
\end{equation*}
\end{lemma}
\begin{lemma}\label{lemma:star_equality}
Let the graphical model on $X_1, X_2, X_3$ and  $X_4$ form a star as shown in Figure \ref{fig:star}. Suppose the bits of each $X_i$ are flipped with probability $q_i<0.5$. Then the following holds:
\begin{equation*}
    \ind{\hat{q}}{1}{2}{3} = \ind{\hat{q}}{1}{2}{4} = \ind{\hat{q}}{1}{4}{3}
\end{equation*}
\end{lemma}
% \textit{Proof.} This is equivalent to proving that
% \begin{equation*}
%     \dfrac{\cov{1}{3}'}{\cov{2}{3}'} = \dfrac{\cov{1}{4}'}{\cov{2}{4}'},  \dfrac{\cov{1}{2}'}{\cov{2}{4}'} = \dfrac{\cov{1}{3}'}{\cov{3}{4}'}
% \end{equation*}
% which is again equivalent to:
% \begin{equation*}
%     \dfrac{\corr{1}{3}}{\corr{2}{3}} = \dfrac{\corr{1}{4}}{\corr{2}{4}},  \dfrac{\corr{1}{2}}{\corr{2}{4}} = \dfrac{\corr{1}{3}}{\corr{3}{4}}.
% \end{equation*}
% Using the correlation decay property it is easy to see that:
% \begin{equation*}
% \begin{aligned}
%     \dfrac{\corr{1}{3}}{\corr{2}{3}} = \dfrac{\corr{1}{4}}{\corr{2}{4}} = \corr{1}{2},
%     \dfrac{\corr{1}{2}}{\corr{2}{4}} = \dfrac{\corr{1}{3}}{\corr{3}{4}}.
% \end{aligned}
% \end{equation*}

The proof of these lemmas and the details of extending these results to generic trees require basic algebraic manipulations and can be found in Appendix \ref{app:thm1_proof}.
\end{proof}
\begin{theorem}\label{thm:unident}
Let $\ee$ denote the probability distribution of $\setx^e$ when the error probability of all the neighbors of leaf nodes is non-zero. 
For any $T'\in \T$,
there exists a set of random variables $\mathcal{X}'$ with conditional independence given by $T'$ and a corresponding error probability vector $\mathbf{\hat{q}}$ such that $\ep = \ee$ where $\ep$ denotes the noisy distribution of $\mathcal{X}'$.
\end{theorem}
We prove this theorem by explicit calculation of $\mathbf{\hat{q}}$. We utilize Lemma \ref{lemma:corr_decay} to enforce the conditional independence relations in any tree $T'\in \T$. The proof is included in Appendix \ref{app:thm2_proof}.

% \end{proof}
Interestingly these unidentifiability results for noisy tree structured Ising models match the ones for noisy tree structured Gaussian graphical models proposed in \cite{pmlr-v97-katiyar19a} inspite of them being graphical models on different class of random variables. 
% Moreover, we remark that the algorithm they propose to learn the equivalence class lacks the sample complexity results as it is in the infinite sample domain. 
We remark that the algorithm they propose to learn the equivalence class in \cite{pmlr-v97-katiyar19a} is in the infinite sample domain and does not provide sample complexity results. 
In the next section we develop an efficient algorithm to learn the equivalence class using samples scaling logarithmically in the number of nodes.

\section{Algorithm}\label{sec:algo}
In this section, we provide a detailed description of our algorithm (Algorithm \ref{alg:full_algo}) that can recover the set $\T$ from noisy samples. 
%The problem of learning $\T$ is equivalent to the problem of learning any tree from $\T$. 
We emphasize that the edges we learn in this algorithm are from any one tree from the equivalence class $\T$ and not necessarily from $T^*$ (as we have demonstrated, identifying $T^*$ itself is not possible). We illustrate our algorithm (as well as each subroutine) using a toy example (Figure \ref{fig:example}).

% \noindent
% \centering

%      \centering
\begin{minipage}{\textwidth}
\begin{minipage}{.45\textwidth}
\begin{algorithm}[H]
    \centering
    \caption{Algorithm to learn the equivalence class $\T$ of the tree}\label{alg:full_algo}
    \footnotesize
\begin{algorithmic}[1]
\Procedure{FindTree}{}
\For{$X_i$ in $\mathcal{X}$}
\State{$EC\gets \fa(X_i, \setx\setminus X_i)$}
\If{$|EC|$ \textgreater 1}
\State{break}
\EndIf
\EndFor
\State{$\mathcal{X}_{last}\gets EC$}
\If{$|\mathcal{X}\setminus \mathcal{X}_{last}| > 0$}
\State{\textsc{Recurse}($X_i, EC[0], \mathcal{X}_{last}$)}
\EndIf
\EndProcedure
\end{algorithmic}
\end{algorithm}
\end{minipage}
\hfill
\begin{minipage}{.5\textwidth}
\begin{algorithm}[H]
    \centering
\captionof{algorithm}{Algorithm to recursively find all the edges within a subtree}\label{alg:recurse}
    \footnotesize
\begin{algorithmic}[1]
\Procedure{Recurse}{$X_i,X_{leaf}, \setx_{last}$}
\State{$subtrees\gets\fb(X_i, X_{leaf}, \setx_{last})$}
% \If{|subtrees| == 0}
% \State{return}
% \EndIf
\For{$subtree$ in $subtrees$}
\State{$EC \gets \fa(X_i, subtree)$}

\State{$\setx_{last}' = \setx_{last}\cup subtrees\setminus subtree \cup EC$}
\State{$\textsc{Recurse}(EC[0], X_i, \setx_{last}')$}
\EndFor
\EndProcedure
\end{algorithmic}
\end{algorithm}
\end{minipage}
\end{minipage}

We first introduce three notions which we extensively use for the algorithm - (i) Equivalence Clusters, (ii) Proximal Sets - $\setp{1}, \setp{2}$, (iii) Categorizing a set of 4 nodes as star/non-star using finite samples.
\paragraph{Equivalence Cluster:} As illustrated in Figure \ref{fig:example}(a), an \textbf{equivalence cluster} is defined as a set containing an internal node and all the leaf nodes connected to it. We denote the equivalence cluster containing a node $X_i$ by $EC(X_i)$.
\paragraph{Defining proximal sets:}
% For a node $X_i$, defining a proximal set allows us to contain the size of the neighborhood to search for nodes in the equivalence cluster of $X_i$. 
Due to correlation decay, the correlation between an arbitrary pair of nodes can go down exponentially in the number of nodes which would lead to exponential sample complexity. To avoid this, we only consider nodes with correlations above a constant threshold while making star/non-star categorization. We define two proximal sets for each node $X_i$-$\setp{1}$ and $\setp{2}$, where $\setp{1}$ is given by:
\begin{equation*}
    \setp{1} = \left\{X_j: \Sigma'_{i,j} \geq t_1\right \}.
\end{equation*}
In the above expression, we set $t_1 = (1-2q_{max})^2(1-\mu_{max}^2)\rho_{min}^4$, thereby guaranteeing that the set contains at least all the nodes within a radius of 4 of that node (on the true unknown graph). This is because the minimum variance of any node is $\sqrt{(1-\mu_{max}^2)}$, the minimum correlation of any 2 nodes at distance 4 in the absence of noise is $\rho_{min}^4$ and the noise can decrease the covariance by a factor of at most $(1-2q_{max})^2$.
% Using the correlation decay property, we know that the correlation between any two nodes at distance 4 is lower bounded by $\tanh^4(\beta)$ before adding noise. The worst case error probabilities for both nodes are $q_{max}$. The constant $0.5$ accounts for the perturbation introduced by  finite sample estimation.

\begin{wrapfigure}{r}{0.6\textwidth}\vspace{-6pt}
\centering
    \includegraphics[scale = 0.35]{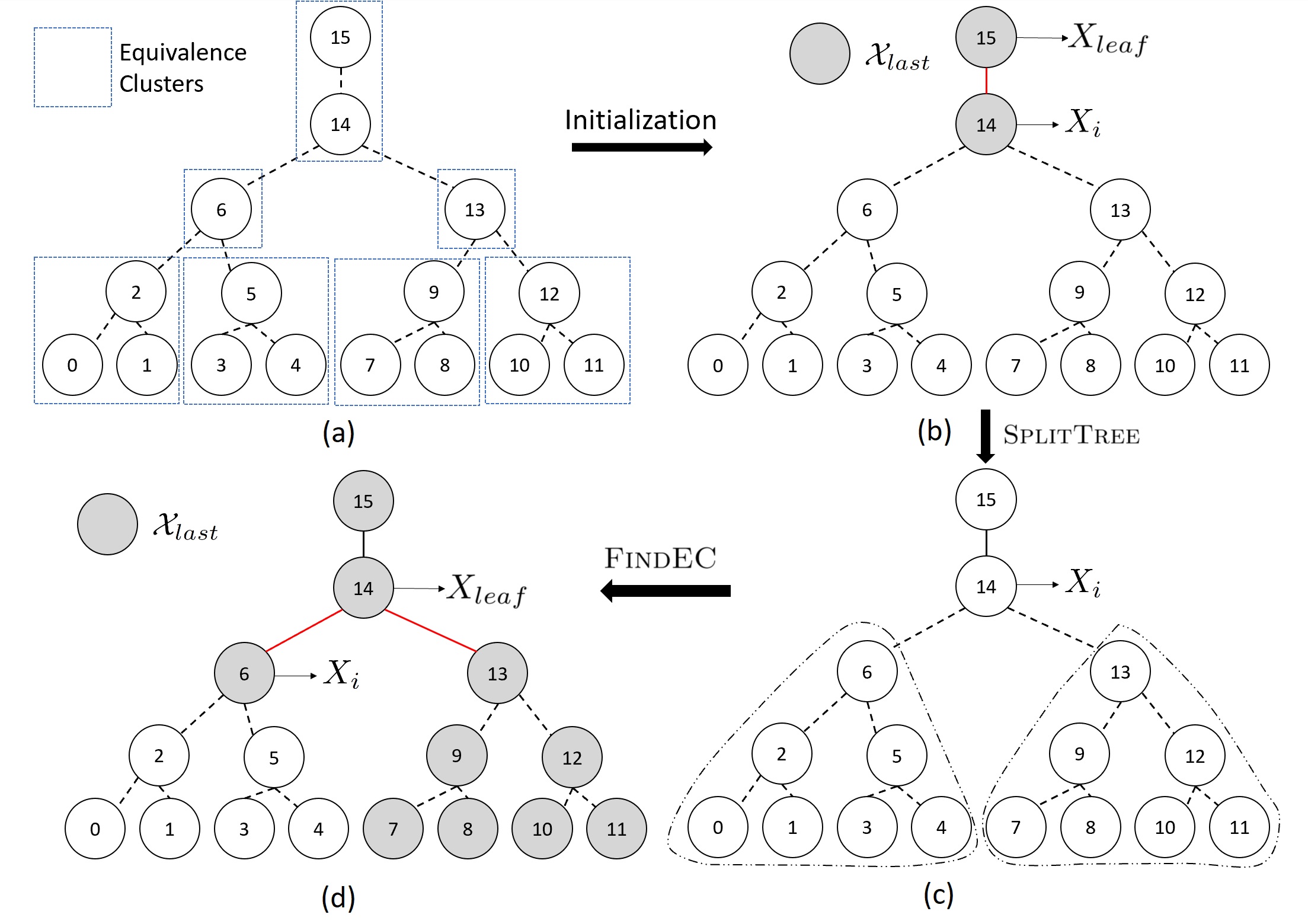}
    \caption{Illustration of the algorithm.}
    \label{fig:example}
\end{wrapfigure}
It is possible that $X_j\in \setp{1}$ even if $X_j$ and $X_i$ are not neighbors in the true graph. We define the second set $\setp{2}$ to guarantee that in this case, at least the first node from the path from $X_i$ to $X_j$ is in $\mathcal{P}_j^2$. Using the correlation decay property, the noisy covariance expression,, and the bounds $1-\mu_{max}^2\leq\Sigma_{ii}\leq 1$, $1-2q_{max}\leq 1-2q_i\leq 1$ and $\rho_{i,j}\leq\rho_{max}$ this set is given by:
\begin{equation*}
    \setp{2} = \left\{X_j: \Sigma'_{i,j} \geq t_2 \right \},
\end{equation*}
where $t_2 = \min\left\{ t_1, \frac{t_1(1-2q_{max})\sqrt{1-\mu_{max}^2}}{\rho_{max}}\right\}$.
Using the sample covariance values we construct sets $\tilde{\setp{1}}$ and $\tilde{\setp{2}}$ to include at least all the nodes in $\setp{1}$ and $\setp{2}$ with high probability. We denote the sample covariance between nodes $X^e_i$ and $X^e_j$ obtained from the noisy observations by $\tilde{\Sigma}_{i,j}$. The sets $\tilde{\setp{1}}$ and $\tilde{\setp{2}}$ are defined as:
\begin{equation*}
    \tilde{\setp{1}} = \left\{X_j^e: \tilde{\Sigma}_{i,j} \geq 0.5t_1\right \},
\end{equation*}
\begin{equation*}
    \tilde{\setp{2}} = \left\{X_j^e: \tilde{\Sigma}_{i,j} \geq 0.5t_2\right \}.
\end{equation*}
% The significance of this set and of radius 4 becomes clear in the description of the corner case of $\fa$.
\paragraph{Classifying into star/non-star shape:}
% \textcolor{red}{Ashish: find latex command to increase cell height in the table below}
It is easy to see (refer Equation(\ref{eq:non_star_ratios}) in Appendix \ref{app:star_non_star_cond}) that when any set of 4 nodes $\{X_1, X_2, X_3, X_4\}$ forms a non-star such that $(X_1, X_2)$ from a pair, we have:
\begin{equation*}
    \dfrac{\corr{1}{3}'\corr{2}{4}'}{\corr{1}{2}'\corr{3}{4}'} \leq \rho_{max}^2,
    \dfrac{\corr{1}{3}'\corr{2}{4}'}{\corr{1}{4}'\corr{2}{3}'} = 1.
\end{equation*}
Therefore, using the finite sample estimate of the noisy correlation, the categorization of  4 nodes $\{X_1, X_2, X_3, X_4\}$ into star/non-star such that $(X_1, X_2)$ from a pair is given in Table \ref{tab:my_label}.

% \twocolumn[]
\aboverulesep=0ex
 \belowrulesep=0ex
 \setlength{\tabcolsep}{20pt}
 \renewcommand{\arraystretch}{2}

% \onecolumn[]
\noindent
We first describe the key steps of the algorithm using the example from Figure \ref{fig:example}.
\paragraph{Initialization:}
\begin{wraptable}{r}{0.6\textwidth}
\centering
 \begin{small}
    \begin{tabular}{|c|c|c|}\toprule
        & \\[-5.5ex]
        Category &$\dfrac{\tilde{\rho}_{1,3}\tilde{\rho}_{2,4}}{\tilde{\rho}_{1,4}\tilde{\rho}_{2,3}}$ &$\dfrac{\tilde{\rho}_{1,3}\tilde{\rho}_{2,4}}{\tilde{\rho}_{1,2}\tilde{\rho}_{3,4}}$\\[1.5ex] \midrule
        & \\[-5.5ex]
        Non-star &$ > \dfrac{(1+\rho_{max}^2)}{2}$ &$ < \dfrac{(1+\rho_{max}^2)}{2}$ \\[1.5ex] \midrule
        & \\[-5.5ex]
        Star     &$ > \dfrac{(1+\rho_{max}^2)}{2}$ & $ > \dfrac{(1+\rho_{max}^2)}{2}$\\[1.5ex] \bottomrule
    \end{tabular}
    \caption{Star/Non-star categorization}
    \label{tab:my_label}
 \end{small}
\end{wraptable}
The algorithm starts by learning edges of any leaf node (edge(14,15)). It searches for a leaf node by looking for equivalence cluster with more than one nodes. This is achieved using the $\fa$ subroutine which, given a node, can find all the nodes in its equivalence cluster. 
\paragraph{Recursion:}
After finding a pair of a leaf and an internal node, the algorithm calls the $\fc$ to find all the remaining edges recursively. This is done in 2 steps: (i) Split the  nodes in the proximal set of the internal node into different subtrees connected to the internal node using $\fb$ (subtree1 = nodes 0-6, subtree2 = nodes 7-13), (ii) Find the nodes in each subtree which have an edge with the internal node. This is done by the simple observation that when we consider the internal node along with one such subtree, the internal node acts like a leaf node for the tree on just these nodes (node 14 is a leaf node for the tree on 1,2,3,4,5,6,14). Hence, $\fa$ can be used on these subset of nodes. Now, we have a pair of a leaf node and an internal node on a subset of nodes and we call the $\fc$ just on this subset of nodes.

The runtime of the algorithm is $\Or(n^3)$.

\subsubsection*{Description of subroutines}
%\textcolor{red}{Vatsal: Check if these descriptions make sense.}
\paragraph{$\fb$}
This function takes in a pair of leaf and internal nodes, $X_{leaf}$ and $X_i$ respectively, from a subtree. To ensure that we only consider nodes in the subtree, it also takes in a set $\setx_{last}$ which contains the nodes assigned to other subtrees in the previous recursive call. It splits the nodes from the common proximal sets  $\tilde{\mathcal{P}^1}_{leaf}$ and $\tilde{\setp{1}}$ in the subtree into smaller subtrees that have an edge with $X_i$. The property used is that 2 nodes $X_{k_1}$ and $X_{k_2}$ lie in the same smaller subtree if and only if $\{X_i, X_{leaf}, X_{k_1}, X_{k_2}\}$ form a non-star such that nodes $(X_{k_1}, X_{k_2})$ form a pair. To make sure that no node outside the subtree is considered, we do not split any node within $\setx_{last}$ or which pairs with a node from $\setx_{last}$ when considered along with $X_{leaf}$ and $X_i$. This is an $\Or(n^2)$ operation.
\paragraph{$\fa$}
This function takes in a node $X_i$ and a set of nodes $\setx_{sub}$ as inputs and returns all the nodes of $EC(X_i)$ from $\setx_{sub}\cup X_i$. Two nodes are in the same equivalence cluster if any categorization of a set of 4 nodes as star/non-star either results in a star or a non-star where these nodes form a pair. Essentially, for all the nodes in the proximal set of $X_i$, we verify if they satisfy this condition. We also add edges by arbitrarily choosing any node from $EC(X_i)$ as the center node as we are only interested in recovering $\T$. This is an $\Or(n^2)$ operation.

\paragraph{$\fc$}
This function takes the same arguments as $\fb$ and calls $\fb$ with those arguments to obtain the different smaller subtrees. It then uses $\fa$ to find the edges between $X_i$ and these smaller subtrees. When the smaller subtrees are considered along with $X_i$, $X_i$ acts as their leaf node and the node $X_i$ has an edge with, acts as an internal node. For each of the smaller subtree, it adds the nodes of the remaining smaller subtrees in $\setx_{last}$ along with the equivalence cluster from the current subtree which had an edge with $X_i$ as that has already been learnt. It then calls the recurse function with these smaller subtrees.

%\textcolor{red}{Vatsal: Check if this makes sense.}
\paragraph{Need for 2 proximal sets $\tilde{\setp{1}}$ and $\tilde{\setp{2}}$:}
Consider the implementation of $\fa$. For each node $X_j\in \tilde{\setp{1}}$, we check if it lies in $EC(X_i)$. For all pairs of nodes $X_{k_1}, X_{k_2}\in \tilde{\setp{1}}\cap \tilde{\mathcal{P}^1_j}$,  we could check whether the condition for  $X_j\in EC(X_i)$ is satisfied. 
%If no node on the path from $X_i$ to $X_j$ lies in $\setp{1}\cap \mathcal{P}^1_j$ for $X_j\not \in EC(X_i)$, then we would incorrectly conclude that $X_j\in EC(X_i)$.
If $X_j\not \in EC(X_i)$ but no node on the path from $X_i$ to $X_j$ is in $\tilde{\setp{1}}\cap \tilde{\mathcal{P}^1_j}$, we would incorrectly conclude that $X_j\in EC(X_i)$. To avoid this corner case, instead of considering $X_{k_1}, X_{k_2}\in \setp{1}\cap \tilde{\mathcal{P}^1_j}$, we consider $X_{k_1}, X_{k_2}\in \tilde{\setp{2}}\cap \tilde{\mathcal{P}^2_j}$. By the definition of $\tilde{\setp{2}}$, $\tilde{\setp{2}}\cap \tilde{\mathcal{P}^2_j}$ contains a node from the path from $X_i$ to $X_j$. This results in $\{X_i, X_j, X_{k_1}, X_{k_2}\}$ forming a non-star where $(X_i, X_j)$ do not form a pair and hence it correctly concludes that $X_j\not \in EC(X_i)$.

$\tilde{\setp{2}}$ plays a crucial role in $\fb$ too. Consider a recursive call to $\fc$ such that in the previous recursion, nodes in $\tilde{\setp{1}}\cap \tilde{\mathcal{P}^1}_{leaf}$ of that step were added to $\setx_{last}$. However, in the current step there might be nodes in $\tilde{\setp{1}}\cap \tilde{\mathcal{P}^1}_{leaf}$ from other subtrees which are not in $\setx_{last}$. To ensure that any such node $X_k$ is not considered while creating smaller subtrees, we consider the nodes $\tilde{\mathcal{P}_k^2}\cap \setx_{last}$ and ensure that they don't pair with $X_k$ when considered with the present recursive call's $X_i$ and $X_{leaf}$.

\paragraph{Setting the radius as 4:}
The radius is set as 4 to make sure that each node is considered with at least its nearest neighboring nodes when classified as star/non-star. This would ideally require a radius of 3 but to account for the unidentifiability between a leaf node and its neighbor, we consider a radius of 4.

The detailed description of the complete algorithm including these subroutines- their pseudo-code, proof of correctness and time complexity analysis is presented in Appendix \ref{app:algo_details}.

Next, we present the sample complexity result of our algorithm.
% \textcolor{red}{Following needs to be rederived}
\begin{theorem} 
Consider an Ising model on $n$ nodes whose graph structure is the tree $T^*$ and it satisfies assumptions \ref{as:bounded_mean} and \ref{as:bounded_corr}. We get noisy samples from this Ising model where samples from each node are flipped independently with unknown, unequal probability satisfying assumption \ref{as:p_e_bound}. Algorithm \ref{alg:full_algo} correctly recovers $\T$ with probability at least $1-\tau$ if the number of samples $m$ is lower bounded as follows:
\begin{equation*}
    m \geq \frac{128}{\delta^2}\log\left(\frac{6n^2}{\tau}\right)
\end{equation*}
where $\delta = \frac{t_2^3(1-t_3)}{128}$, $t_2 = \min\left\{t_1, \frac{t_1(1-2q_{max})\sqrt{1-\mu_{max}^2}}{\rho_{max}}\right\}$, $t_1 = (1-2q_{max})^2(1-\mu_{max}^2)\rho_{min}^4$, $t_3 = \frac{1+\rho_{max}^2}{2}$.

\end{theorem}
The proof is included in Appendix \ref{app:sample_complexity}.

\paragraph{Running the algorithm in absence of the knowledge of $\rho_{max}$ and $\rho_{min}$:}
When we only have access to noisy samples, $q_{max}$ and $\mu_{max}$, we can set $\rho_{min} = \epsilon$ and $\rho_{max} = 1-\epsilon$ and solve for $\epsilon$ given an error budget $\tau$. This would be a conservative estimate of $\epsilon$ and in practice, the algorithm would work for lower values of $\epsilon$.
\section{Experiments}\label{app:experiments}\label{sec:exp}

\subsection*{Experimental Setup}
The probability distribution of Ising models is represented by 
\begin{equation*}
    \mathbb{P}(\bs{{X}} = \bs{x})\propto \exp(\bs{x}^\top \bs{W} \bs{x}/2 + \bs{b}^\top\bs{x}),
\end{equation*}
where $\bs{W}$ is the symmetric weight matrix with 0 on the diagonal and $\bs{b}$ is the bias vector. The support of $\bs{W}$ defines the Ising model structure. Therefore, a chain structured Ising model with nodes labeled consecutively satisfies $W_{i,j}\neq 0$ if and only if $|i-j| = 1$. A star structured Ising model with node 1 as the internal node satisfies $W_{i,j}\neq 0$ if and only if $i = 1, j\neq 1$ or $j = 1, i\neq 1$.
\begin{wrapfigure}{r}{0.5\textwidth}\vspace{-0.5cm}
\centering
    \includegraphics[height=50mm, width=0.48\textwidth]{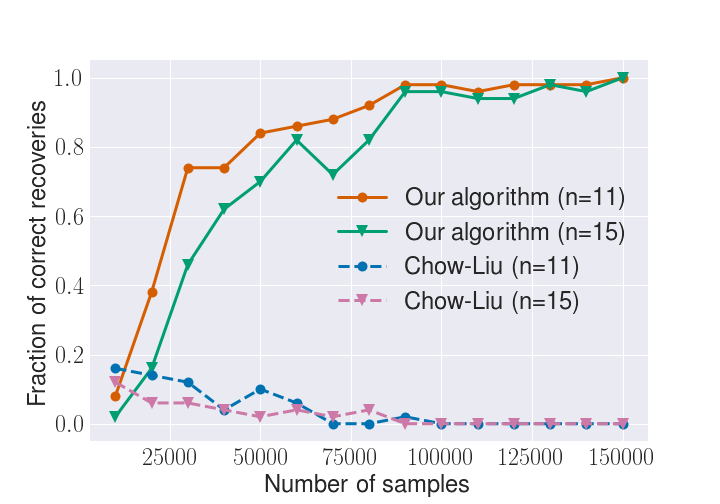}
    \caption{$ W_{min}=0.7,~ W_{max}=1.2, ~q_{max}=0.15$. Comparing the performance of Chow-Liu and our algorithm over 50 runs for increasing number of samples.}
    \label{fig:plots1}
\end{wrapfigure}
We sample each non-zero element of $\bs{W}$ uniformly from $[W_{min}, W_{max}]$. The probability of error for each node is sampled uniformly from $[0, q_{max}]$.

Comparison with the Chow-Liu algorithm is presented in \ref{subsec:our_vs_CL}. In subsection \ref{subsec:our_vs_spars}, we demonstrate the performance of our algorithm as compared to the Sparsitron algorithm. 
Subsections \ref{subsec:q_max}, \ref{subsec:W_max} and \ref{subsec:W_min} report the impact of the maximum error $q_{max}$, maximum edge weight $W_{max}$ and minimum edge weight $W_{min}$ respectively on the algorithm performance. All of these experiments are done for chain structured Ising models and have $\bs{b} = 0$.

We next demonstrate the performance of our algorithm on star-structured Ising model in subsection \ref{subsec:star} as the number of nodes increases. To begin with, we stick to $\bs{b} = 0$.

Finally we study the impact of $\bs{b}$ (the non-zero mean) on the algorithm performance for both star-structured Ising model and chain structured Ising model in subsection \ref{subsec:b}.
\begin{wrapfigure}{r}{0.5\textwidth}\vspace{-0.5cm}
    \centering
    \includegraphics[scale = 0.3]{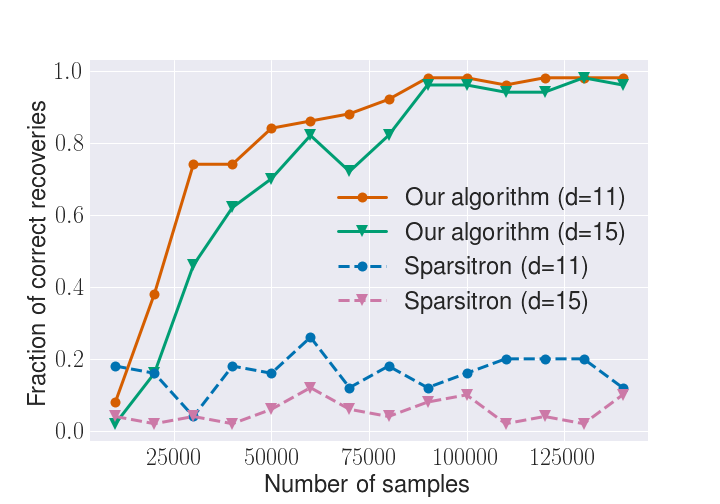}
    \caption{$ W_{min}=0.7, W_{max}=1.2, q_{max}=0.15$. Comparing the performance of Sparsitron and our algorithm over 50 runs for increasing number of samples.}
    \label{fig:plot_spars}\vspace{-1.5cm}
\end{wrapfigure}
\subsection{Our Algorithm vs Chow-Liu}\label{subsec:our_vs_CL}
Figure \ref{fig:plots1} depicts that as the number of samples increases, the fraction of correct recoveries using our algorithm approaches 1. Chow-Liu incorrectly converges to a tree not in the equivalent class $\T$ due to unequal noise in the nodes. 

\subsection{Our Algorithm vs Sparsitron}\label{subsec:our_vs_spars}
\begin{wrapfigure}{r}{0.5\textwidth}\vspace{-0.5cm}
    \centering
    \includegraphics[scale = 0.3]{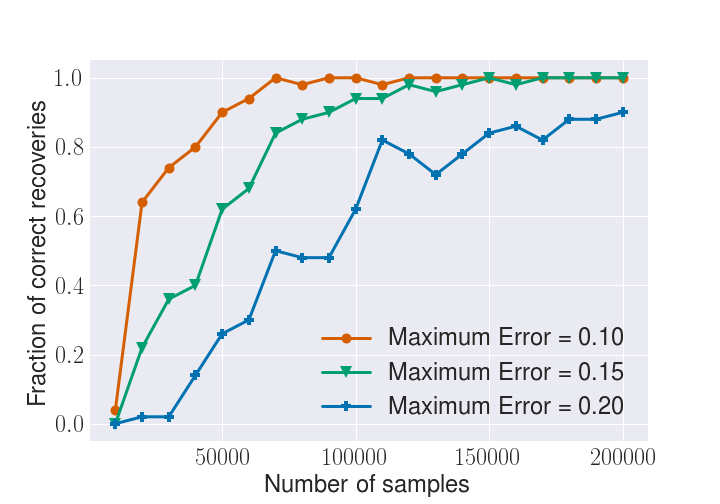}
    \caption{$ W_{min}=0.7,~ W_{max}=1.2$. The performance of our algorithm on a chain of 15 nodes over 50 runs for increasing number of samples for different values of $q_{max}$. }
    \label{fig:plot_q_max}\vspace{-0.4cm}
\end{wrapfigure}
% \begin{figure}[!htb]
%     \centering
%     \subfloat[]{\includegraphics[width = \linewidth, height = 0.75\linewidth]{Figs/chained_sparsitron_d11_15.png}}
%     \caption{$ W_{min}=0.7, W_{max}=1.2, q_{max}=0.15$. Comparing the performance of Sparsitron and our algorithm over 50 runs for increasing number of samples. }
%     \label{fig:plot_spars}
% \end{figure}
We set $ W_{min}=0.7, W_{max}=1.2, q_{max}=0.15, \bs{b} = 0$.
In Figure \ref{fig:plot_spars}, we compare the performance of our algorithm with the sparsitron algorithm for chain-structured Ising model on 15 nodes. To evaluate the sparsitron algorithm, we take the output weight matrix and find the maximum weight spanning tree. We call the algorithm a success if this tree is from the equivalence class $\T$. We can see that the sparsitron has a low success rate.

\subsection{Effect of the Maximum Error Probability} \label{subsec:q_max}
We set $ W_{min}=0.7, W_{max}=1.2, \bs{b} = 0$ and vary $q_{max}$ to take the values $\{0.1, 0.15, 0.20\}$. We evaluate on chain structured Ising model on 15 nodes. Figure \ref{fig:plot_q_max} illustrates the effect of $q_{max}$ on the performance of our algorithm. As expected, increase in the probability of error makes it harder to recover the equivalence class.
\vspace{-0.5cm}

% \begin{figure}[!htb]
%     \centering
%     \subfloat[]{\includegraphics[width = \linewidth, height = 0.75\linewidth]{Figs/chained_W_max.png}}
%     \caption{$ W_{min}=0.7, ~q_{max}=0.15$. The performance of our algorithm  on a chain of 15 nodes over 50 runs for increasing number of samples for different values of $W_{max}$. }
%     \label{fig:plot_W_max}
% \end{figure}

\vspace{-0.3cm}
\subsection{Effect of the Maximum Weight} \label{subsec:W_max}
Next, we look at the effect of $W_{max}$. We set $ W_{min}=0.7, q_{max}=0.15, \bs{b} = 0$ and vary $W_{max}$ to take the values $\{1.0, 1.2, 1.4\}$. We evaluate on chain structured Ising model on 15 nodes. 
Intuitively, a high $W_{max}$ makes it difficult to differentiate between a star and a non-star, therefore the algorithm is expected to perform better for lower $W_{max}$. This is indeed what happens as shown in Figure \ref{fig:plot_W_max}.
\begin{minipage}{\textwidth}
\begin{minipage}{0.45\textwidth}
\begin{figure}[H]
    \centering
    \includegraphics[scale = 0.3]{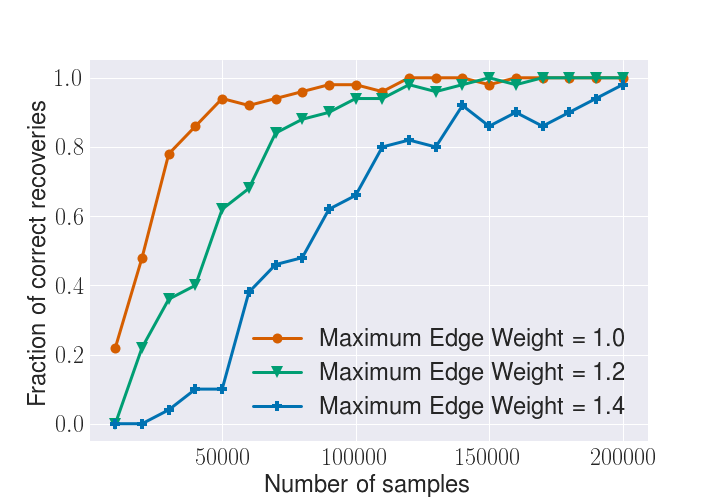}
    \caption{$ W_{min}=0.7, ~q_{max}=0.15$. The performance of our algorithm  on a chain of 15 nodes over 50 runs for increasing number of samples for different values of $W_{max}$. }
    \label{fig:plot_W_max}
\end{figure}
\end{minipage}
\begin{minipage}{0.03\textwidth}
~
\end{minipage}
\begin{minipage}{0.45\textwidth}
\begin{figure}[H]
    \centering
    \includegraphics[scale = 0.3]{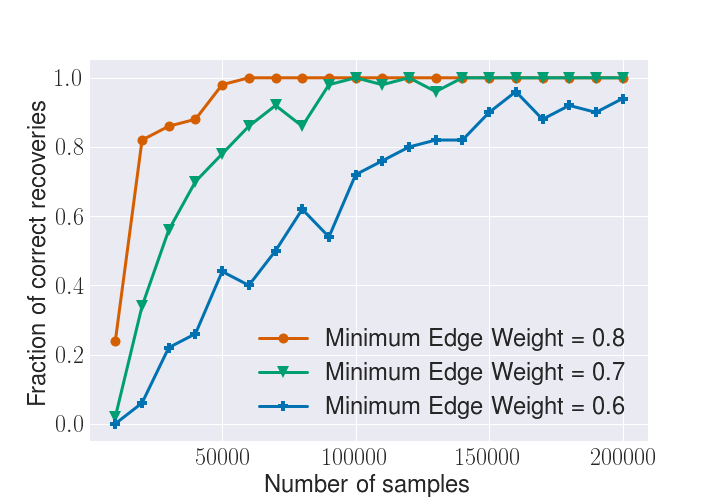}
    \caption{$ W_{max}=1.1, ~q_{max}=0.15$. The performance of our algorithm  on a chain of 15 nodes over 50 runs for increasing number of samples for different values of $W_{min}$. }
    \label{fig:plot_W_min}
\end{figure}
\end{minipage}
\end{minipage}

\subsection{Effect of the Minimum Weight}\label{subsec:W_min}
% \begin{figure}
%     \centering
%     \includegraphics[scale = 0.35]{Figs/chained_W_min.png}
%     \caption{$ W_{max}=1.1, ~q_{max}=0.15$. The performance of our algorithm  on a chain of 15 nodes over 50 runs for increasing number of samples for different values of $W_{min}$. }
%     \label{fig:plot_W_min}
% \end{figure}
% \begin{figure}[!htb]
%     \centering
%     \subfloat[]{\includegraphics[width = \linewidth, height = 0.75\linewidth]{Figs/chained_W_min.png}}
%     \caption{$ W_{max}=1.1, ~q_{max}=0.15$. The performance of our algorithm  on a chain of 15 nodes over 50 runs for increasing number of samples for different values of $W_{min}$. }
%     \label{fig:plot_W_min}
% \end{figure}

We also look at the effect of $W_{min}$. We set $ W_{max}=1.1, q_{max}=0.15, \bs{b} = 0$ and vary $W_{min}$ to take the values $\{0.6, 0.7, 0.8\}$. We evaluate on chain structured Ising model on 15 nodes. We can expect that smaller edge weights make it difficult to accurately estimate the correlation values resulting in higher errors. This is illustrated in Figure \ref{fig:plot_W_min}.
\vspace{-0.3cm}
\subsection{Algorithm performance for Star-structured tree Ising Model} \label{subsec:star}

% \begin{figure}[!htb]
%     \centering
%     \subfloat[]{\includegraphics[width = \linewidth, height = 0.75\linewidth]{Figs/star_our_algo_d11_13_15.png}}
%     \caption{$ W_{min} = 0.7, W_{max}=1.2, ~q_{max}=0.1$. The performance of our algorithm  on a star of 11, 13 and 15 nodes over 50 runs for increasing number of samples. }
%     \label{fig:res_star}
% \end{figure}

Till now all the experiments have used chain structured Ising models. In this subsection, we consider the other extreme tree structure-star structure with one internal node and the remaining leaf nodes connected to the internal node. We set $ W_{min}=0.7, W_{max}=1.2, q_{max}=0.1, \bs{b} = 0$. We consider three different graph sizes - 11 nodes, 13 nodes and 15 nodes. The performance of the algorithm is illustrated in Figure \ref{fig:res_star}.
\begin{wrapfigure}{r}{0.50\textwidth}\vspace{-0.5cm}
    \centering
    \includegraphics[scale = 0.3]{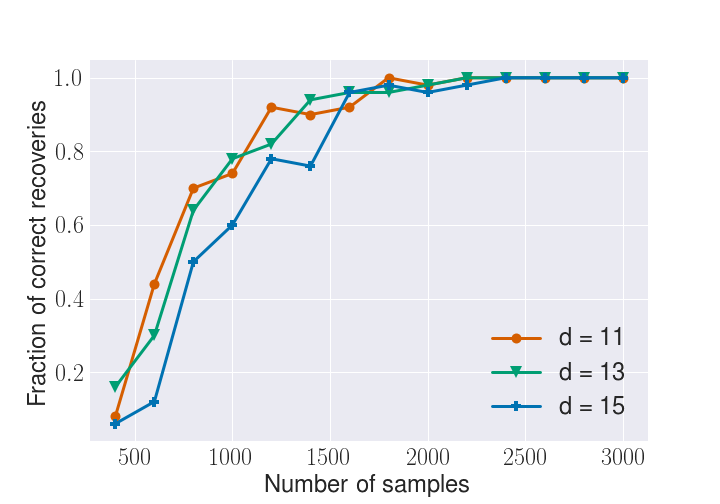}
    \caption{$ W_{min} = 0.7, W_{max}=1.2, ~q_{max}=0.1$. The performance of our algorithm  on a star of 11, 13 and 15 nodes over 50 runs for increasing number of samples. }
    \label{fig:res_star}\vspace{-1.2cm}
\end{wrapfigure}
As we can see, compared to the chain structured Ising models, star-structured Ising models require much smaller number of samples. This is due to the small radius which results in high absolute values of the correlations. Also, the performance is only slightly impacted by the number of nodes which can also be attributed to the small radius.

\subsection{Effect of bias on the algorithm performance.} \label{subsec:b}
We observe a very interesting phenomena when we study the impact of the bias term $\bs{b}$. For chain structured Ising model, the algorithm performs better for lower absolute values in $\bs{b}$. However, for star structured Ising modes, the algorithm performs better for higher absolute values in $\bs{b}$. For both the cases we consider a tree on 11 nodes and set $ W_{min}=0.7, W_{max}=1.2, q_{max}=0.1$. $\mu_{max}$ is estimated empirically. We set all the entries in $\bs{b}$ to be equal. 

For chain structured Ising model, we consider 3 cases - (i) all the entries in $\bs{b}$ = 0.0, (ii) all the entries in $\bs{b}$ = 0.02 and (iii) all the entries in $\bs{b}$ = 0.04. The result is provided in Figure \ref{fig:bias_chain}.
\begin{minipage}{\textwidth}
\begin{minipage}{0.45\textwidth}
\begin{figure}[H]
    \centering
    \includegraphics[scale = 0.3]{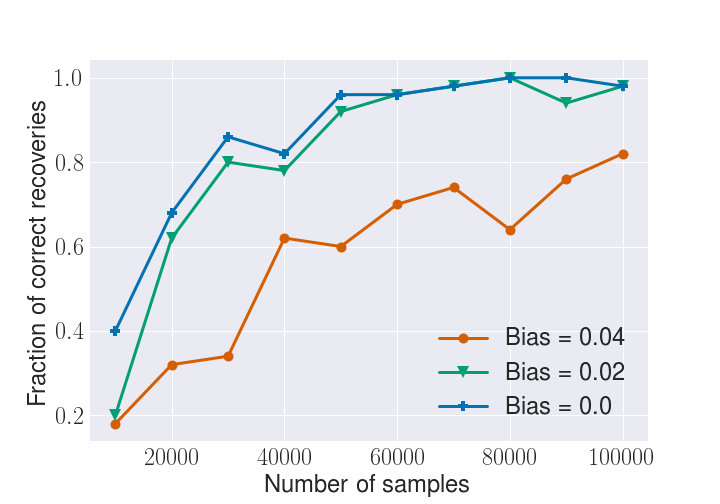}
    \caption{$ W_{min} = 0.7, W_{max}=1.2, ~q_{max}=0.1$. The performance of our algorithm  on a chain of 11 nodes over 50 runs for increasing number of samples for varying bias values. }
    \label{fig:bias_chain}
\end{figure}
\end{minipage}
\begin{minipage}{0.03\textwidth}
~
\end{minipage}
\begin{minipage}{0.45\textwidth}
\begin{figure}[H]
    \centering
    \includegraphics[scale = 0.3]{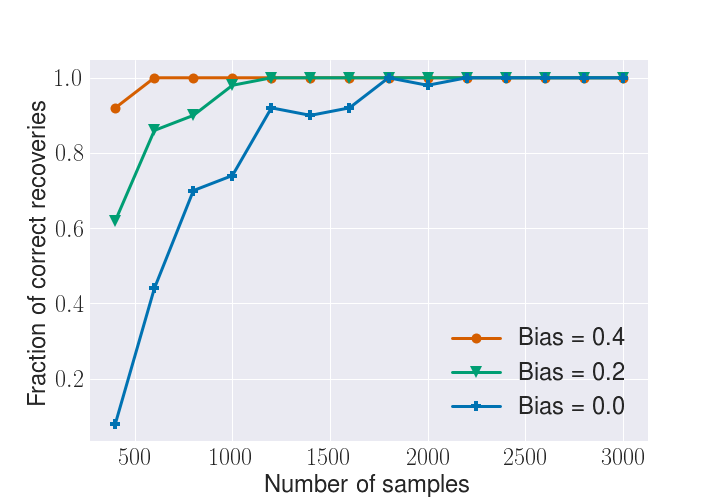}
    \caption{$ W_{min} = 0.7, W_{max}=1.2, ~q_{max}=0.1$. The performance of our algorithm  on a star of 11 nodes over 50 runs for increasing number of samples for varying bias values. }
    \label{fig:bias_star}
\end{figure}
\end{minipage}
\end{minipage}

\vspace{1em}
For star structured Ising model, we consider 3 cases - (i) all the entries in $\bs{b}$ = 0.0, (ii) all the entries in $\bs{b}$ = 0.2 and (iii) all the entries in $\bs{b}$ = 0.4. The result is provided in Figure \ref{fig:bias_star}.

We now give an intuitive justification of the observation.
A higher value of bias results in smaller threshold for the proximal sets. For chain structured Ising models, due to correlation decay, we can expect correlation values for some pairs of random variables to be close to the threshold. Estimating them accurately would require more number of samples.

Star structured Ising models, on the other hand, have higher correlation values due a smaller diameter of 2. A low threshold ensures that no node is mistakenly left behind when constructing the proximal sets. Therefore, the performance is better for a higher bias values.
\appendix
\section{Proof of Lemma \ref{lemma:corr_decay}}\label{app:corr_decay}
We prove this by induction on the number of nodes $k$ in the path $(X_{i_1}\rightarrow X_{i_2} \rightarrow X_{i_3}\dots  \rightarrow X_{i_k})$ for any 2 nodes $X_{i_1}, X_{i_k}$.
\subsubsection*{Base Case $k = 3$:}
The path is $(X_{i_1}\rightarrow X_{i_2} \rightarrow X_{i_3})$, therefore we have $X_{i_1}\perp X_{i_3}|X_{i_2}$. For random variables with a support size of 2, this is true if and only if they are conditionally uncorrelated, that is,
\begin{equation}\label{eq:corr_dec_uncorr}
   \E{X_{i_1}X_{i_3}|X_{i_2}} = \E{X_{i_1}|X_{i_2}}\E{X_{i_3}|X_{i_2}}.
\end{equation}
$\E{X_{i_1}|X_{i_2}}$ is linear in $X_{i_2}$ since the support size of $X_{i_2}$ is 2 and therefore we need to need to fit only 2 points $\E{X_{i_1}|X_{i_2} = 1}$ and $\E{X_{i_1}|X_{i_2} = -1}$ to completely represent the conditional expectation. Therefore the linear least square error (LLSE) estimator of $X_{i_1}$ given $X_{i_2}$ is also the minimum mean squared estimator $\E{X_{i_1}|X_{i_2}}$. Utilizing the standard result for LLSE, we have:
\begin{equation}\label{eq:corr_dec_cond_exp}
    \E{X_{i_1}|X_{i_2}} = \E{X_{i_1}} + \Sigma_{i_1, i_2}\Sigma_{i_2, i_2}^{-1}(X_{i_2} - \E{X_{i_2}}).
\end{equation}
Similarly we have:
\begin{equation}\label{eq:corr_dec_cond_exp1}
    \E{X_{i_3}|X_{i_2}} = \E{X_{i_3}} + \Sigma_{i_3, i_2}\Sigma_{i_2, i_2}^{-1}(X_{i_2} - \E{X_{i_2}}).
\end{equation}
Substituting $\E{X_{i_1}|X_{i_2}}$ and $\E{X_{i_3}|X_{i_2}}$ from Equations (\ref{eq:corr_dec_cond_exp}) and (\ref{eq:corr_dec_cond_exp1}) in Equation (\ref{eq:corr_dec_uncorr}) we get:
\begin{equation*}
\begin{aligned}
    \E{X_{i_1}X_{i_3}|X_{i_2}} =& \E{X_{i_1}}\E{X_{i_3}} + \E{X_{i_1}}\Sigma_{i_3, i_2}\Sigma_{i_2, i_2}^{-1}(X_{i_2} - \E{X_{i_2}}) +\\
    &\E{X_{i_3}}\Sigma_{i_1, i_2}\Sigma_{i_2, i_2}^{-1}(X_{i_2} - \E{X_{i_2}}) +\\
    &\Sigma_{i_1, i_2} \Sigma_{i_3, i_2}(\Sigma_{i_2, i_2}^{-1}(X_{i_2} - \E{X_{i_2}}))^2\\
    \E{X_{i_1}X_{i_3}} =& \E{\E{X_{i_1}X_{i_3}|X_{i_2}}}\\
    =& \E{X_{i_1}}\E{X_{i_3}} + \Sigma_{i_1, i_2} \Sigma_{i_3, i_2}\Sigma_{i_2, i_2}^{-1}.\\
\end{aligned}
\end{equation*}
Therefore we get $\Sigma_{i_1, i_3}\Sigma_{i_2, i_2} = \Sigma_{i_1, i_2}\Sigma_{i_3, i_2}$ which implies $\rho_{i_1i_3} = \rho_{i_1i_2}\rho_{i_2i_3}$.
\subsubsection*{Inductive Case:}
Let the statement be true for any path involving $k$ nodes. For a path $(X_{i_1}\rightarrow X_{i_2} \rightarrow X_{i_3}\dots  \rightarrow X_{i_{(k+1)}})$ we have $X_{i_1}\perp X_{i_{(k+1)}}|X_{i_k}$. Therefore the same calculation as the base case holds true by replacing $X_{i_2}$ by $X_{i_k}$ and $X_{i_3}$ by $X_{i_{(k+1)}}$. Therefore $\rho_{i_1i_{(k+1)}} = \rho_{i_1i_k}\rho_{i_ki_{(k+1)}}$. By the inductive assumption, $\rho_{i_1i_k} = \prod_{l=2}^{k}\rho_{i_{l-1},i_l}$, therefore, $\rho_{i_1i_{(k+1)}} = \prod_{l=2}^{k+1}\rho_{i_{l-1},i_l}$.
\section{Proof of Covariance of noisy variables.}\label{app:cov_noisy}
\begin{lemma}
Consider 2 Random variables $X_i$ and $X_j$ with support on $\{-1, 1\}$ whose covariance is denoted by $\Sigma_{i,j}$. Now consider the noisy version of these random variables $X_i^e$ and $X_j^e$ whose covariance is denoted by $\Sigma_{i,j}'$. Then we have:
\begin{align*}
    \E{X_i^e} &= (1-2q_i)\E{X_i}\\
    \Sigma_{i,j}' &= (1-2q_i)(1-2q_j)\Sigma_{i,j}
\end{align*}
\end{lemma}
\begin{proof}
By the noise model we have:
\begin{equation}\label{eq:noisy_mean}
\begin{aligned}
    \E{X_i^e} &= (1-q_i)\E{X_i} + q_i\E{-X_i}\\
    \E{X_i^e} &= (1-2q_i)\E{X_i}.
\end{aligned}
\end{equation}
We also have:
\begin{equation}
\begin{aligned}
    \E{X_i^eX_j^e} =& (1-q_i)(1-q_j)\E{X_iX_j} + (1-q_j)q_i\E{-X_iX_j} + \\
    &(1-q_i)q_j\E{-X_iX_j} + q_iq_j\E{X_iX_j}\\
    =&(1-2q_i)(1-2q_j)\E{X_iX_j}.
\end{aligned}
\end{equation}
Therefore, 
\begin{equation}\label{eq:noisy_covar}
    \begin{aligned}
        \Sigma_{i,j}' &= \E{X_i^eX_j^e} - \E{X_i^e}\E{X_j^e}\\
        &=(1-2q_i)(1-2q_j)(\E{X_iX_j} - \E{X_i}\E{X_j})\\
        &=(1-2q_i)(1-2q_j)\Sigma_{i,j}
    \end{aligned}
\end{equation}
\end{proof}
We can use Equation (\ref{eq:noisy_mean}) to calculate the variance of every random variable in terms of the variance of its noisy counterpart as follows:
\begin{equation}\label{eq:noisy_var}
\begin{aligned}
    \Sigma_{i,i} =& 1 - \E{X_i}^2\\
    =& 1 - \frac{\E{X_i^e}^2}{(1-2q_i)^2}\\
    =& 1 - \frac{1 - \Sigma'_{i,i}}{(1-2q_i)^2}
    \end{aligned}
\end{equation}

\section{Proof that the Quadratic gives a valid solution}\label{sec:valid_sol}
Consider the quadratic in Equation (\ref{eq:quadratic}). We prove that this equation always has a valid solution $q_1<0.5$ for any set of 3 nodes in a tree structured graphical model.

Whenever $0<1 - \Sigma'_{1,1} + \frac{\Sigma'_{1,2}\Sigma'_{1,3}}{\Sigma'_{2,3}}<1$, the solution is of the form $q_1 = \eta, 1-\eta$ where $0\leq\eta<0.5$. From Equations (\ref{eq:noisy_covar}) and (\ref{eq:noisy_var}), we have:
\begin{equation}
    1 - \Sigma'_{1,1} + \frac{\Sigma'_{1,2}\Sigma'_{1,3}}{\Sigma'_{2,3}} = (1-2q_1^2)(1-\Sigma_{1,1}+\frac{\Sigma_{1,2}\Sigma_{1,3}}{\Sigma_{2,3}}).
\end{equation}

The different possible configurations of any 3 nodes $X_1$, $X_2$ and $X_3$ in any tree structured graphical model are shown in Figure \ref{fig:configs}.
\begin{figure}
    \centering
    \includegraphics[scale = 0.35]{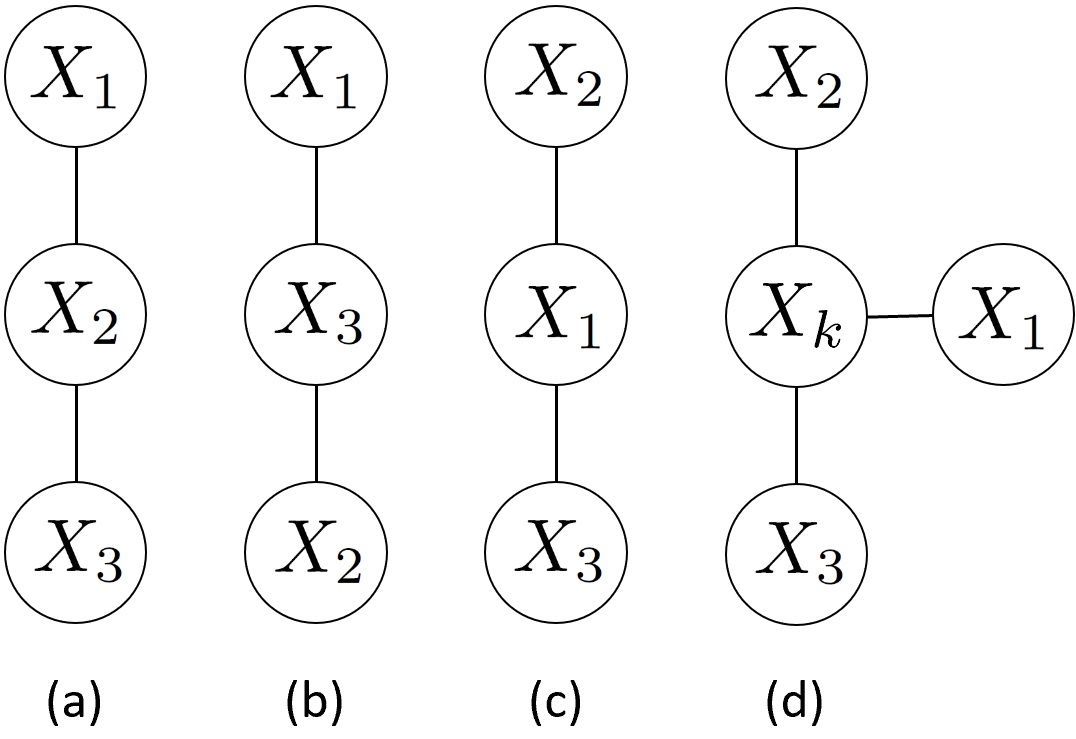}
    \caption{Different possible configurations of any set of 3 nodes.}
    \label{fig:configs}
\end{figure}
For case (a) we have $\Sigma_{2,2} \Sigma_{1,3}= \Sigma_{1,2}\Sigma_{2,3}$ by Lemma \ref{lemma:corr_decay}. This gives us:
\begin{equation*}
    1 - \Sigma'_{1,1} + \frac{\Sigma'_{1,2}\Sigma'_{1,3}}{\Sigma'_{2,3}} = (1-2q_1)^2(1-\Sigma_{1,1} + \frac{\Sigma_{1,2}^2}{\Sigma_{2,2}}).
\end{equation*}
Using the assumption that the absolute value of correlation is upper bounded away from 1 and lower bounded away from 0, we have $0<\Sigma_{1,2}^2 < \Sigma_{1,1}\Sigma_{2,2}$. Also, $0<\Sigma_{1,1}\leq 1$ and $0<(1-2q_1)^2\leq1$. Therefore, for case (a), 
$0< 1 - \Sigma'_{1,1} + \frac{\Sigma'_{1,2}\Sigma'_{1,3}}{\Sigma'_{2,3}} < 1$ and the quadratic equation has valid roots. By symmetry, the quadratic equation gives valid roots for case (b) too.

Case (c) is the underlying truth, therefore the quadratic equation recovers the true underlying error.

For case(d), we have $\Sigma_{k,k}\Sigma_{1,3} = \Sigma_{1,k}\Sigma_{3,k}$, $\Sigma_{k,k}\Sigma_{1,2} = \Sigma_{1,k}\Sigma_{2,k}$ and $\Sigma_{k,k}\Sigma_{2,3} = \Sigma_{2,k}\Sigma_{3,k}$. This gives us:
\begin{equation}
   1 - \Sigma'_{1,1} + \frac{\Sigma'_{1,2}\Sigma'_{1,3}}{\Sigma'_{2,3}} = (1-2q_1)^2(1-\Sigma_{1,1} + \frac{\Sigma_{1,k}^2}{\Sigma_{k,k}}).
\end{equation}
The same arguments as case (a) hold true for case (d) with node 2 replaced by node $k$. Therefore, the quadratic has a valid solution in this case too.
\section{Proof of Lemma \ref{lemma:equality_inequality}, Lemma \ref{lemma:star_equality} and Star/Non-star Condition for Generic Trees}\label{app:thm1_proof}
\subsection{Proof of Lemma \ref{lemma:equality_inequality}(a)}
\textit{Proof.} Note that $\hat{q}_1^{2,3}$ and $\hat{q}_1^{2,4}$ are given by solving an equation similar to (\ref{eq:quadratic}). As the solution to the quadratic is defined completely by the covariance terms, all we need to prove is:
\begin{equation*}
\begin{aligned}
    \dfrac{\cov{1}{2}'\cov{1}{3}'}{\cov{2}{3}'}= \dfrac{\cov{1}{2}'\cov{1}{4}'}{\cov{2}{4}'}\iff
    \dfrac{\cov{1}{3}'}{\cov{2}{3}'}= \dfrac{\cov{1}{4}'}{\cov{2}{4}'}.
\end{aligned}
\end{equation*}
By substituting the value of $\cov{i}{j}'$ from Equation \ref{eq:noisy_covar}, we now need to prove that:
\begin{equation*}
\dfrac{\cov{1}{3}}{\cov{2}{3}} = \dfrac{\cov{1}{4}}{\cov{2}{4}}\iff \dfrac{\corr{1}{3}}{\corr{2}{3}} = \dfrac{\corr{1}{4}}{\corr{2}{4}}.
\end{equation*}
Using the correlation decay property, we get that $\corr{1}{3} = \corr{1}{2}\corr{2}{3}$, $\corr{1}{4} =\corr{1}{2}\corr{2}{3}\corr{3}{4}$ and $\corr{2}{4} = \corr{2}{3}\corr{3}{4}$. Therefore LHS = RHS = $\corr{1}{2}$.

\subsection{Proof of Lemma \ref{lemma:equality_inequality}(b)}
\textit{Proof.} Using the same arguments as in the proof of Lemma \ref{lemma:equality_inequality}(a), we can conclude that we need to prove:
\begin{equation*}
    \dfrac{\cov{1}{3}'\cov{2}{4}'}{\cov{2}{1}'\cov{3}{4}'}\neq 1, \dfrac{\cov{2}{3}'\cov{1}{4}'}{\cov{1}{2}'\cov{3}{4}'}\neq 1
\end{equation*}
Substituting $\cov{i}{j}'$ from Equation (\ref{eq:noisy_covar}), we get:
\begin{equation*}
    \dfrac{\cov{1}{3}'\cov{2}{4}'}{\cov{2}{1}'\cov{3}{4}'} = \dfrac{\corr{1}{3}\corr{2}{4}}{\corr{2}{1}\corr{3}{4}}, \dfrac{\cov{2}{3}'\cov{1}{4}'}{\cov{1}{2}'\cov{3}{4}'} = \dfrac{\corr{2}{3}\corr{1}{4}}{\corr{1}{2}\corr{3}{4}}.
\end{equation*}
Using the correlation decay property, we get that:
\begin{equation}\label{eq:ratio}
 \dfrac{\corr{1}{3}\corr{2}{4}}{\corr{2}{1}\corr{3}{4}}= \dfrac{\corr{2}{3}\corr{1}{4}}{\corr{1}{2}\corr{3}{4}} = \corr{2}{3}^2 \leq \rho_{max}^2 < 1   
\end{equation}

\subsection{Proof of Lemma \ref{lemma:star_equality}}
\textit{Proof.} This is equivalent to proving that
\begin{equation*}
    \dfrac{\cov{1}{3}'}{\cov{2}{3}'} = \dfrac{\cov{1}{4}'}{\cov{2}{4}'},  \dfrac{\cov{1}{2}'}{\cov{2}{4}'} = \dfrac{\cov{1}{3}'}{\cov{3}{4}'}
\end{equation*}
which is again equivalent to:
\begin{equation*}
    \dfrac{\corr{1}{3}}{\corr{2}{3}} = \dfrac{\corr{1}{4}}{\corr{2}{4}},  \dfrac{\corr{1}{2}}{\corr{2}{4}} = \dfrac{\corr{1}{3}}{\corr{3}{4}}.
\end{equation*}
Using the correlation decay property it is easy to see that:
\begin{equation*}
\begin{aligned}
    \dfrac{\corr{1}{3}}{\corr{2}{3}} = \dfrac{\corr{1}{4}}{\corr{2}{4}} = \corr{1}{2},
    \dfrac{\corr{1}{2}}{\corr{2}{4}} = \dfrac{\corr{1}{3}}{\corr{3}{4}}.
\end{aligned}
\end{equation*}

\subsection{Proof of Star/Non-star Condition for Generic Trees}\label{app:star_non_star_cond}
We show how to utilize the result on a set of 4 nodes to classify any set of 4 nodes as star/non-star in a generic tree.

If any 4 nodes $\{X_1, X_2, X_3, X_4\}$ in a tree graphical model form a non-star shape such that $(X_1, X_2)$ from a pair and are not arranged in a chain, there exist nodes $X_k$ and $X_{k'}$ such that the conditional independence structure is given by either Figure \ref{fig:non_star}(a) or \ref{fig:non_star}(b).
\begin{figure}
    \centering
    \includegraphics[scale=0.35]{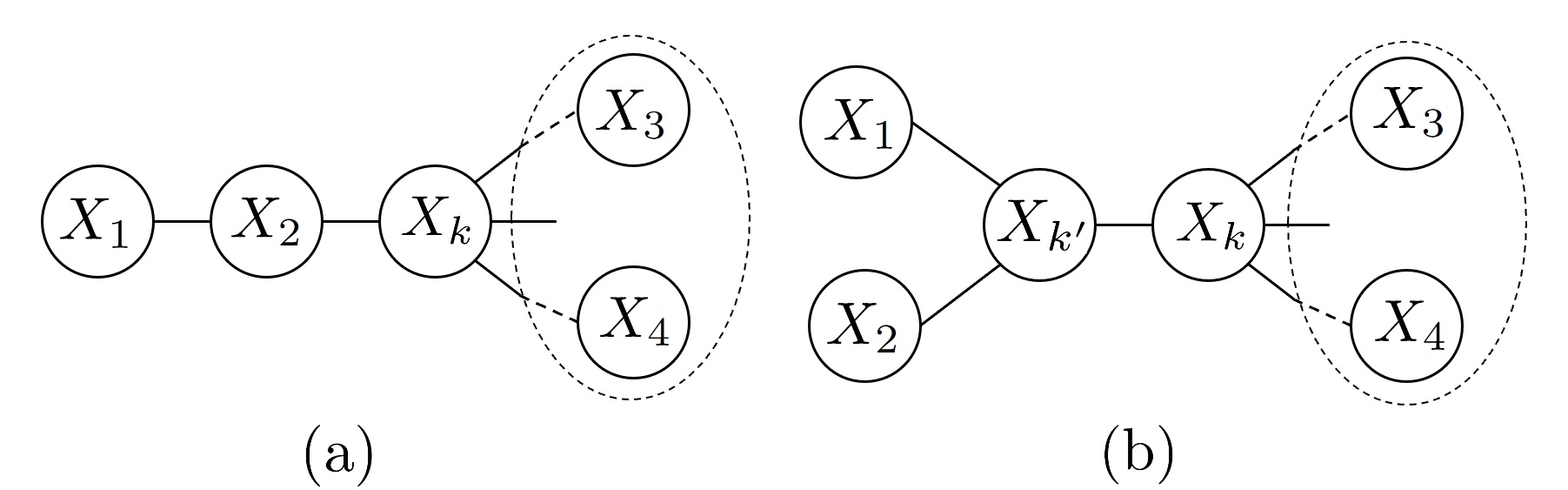}
    \caption{Possible conditional independence relations for non-star shape if they don't form a chain}
    \label{fig:non_star}
\end{figure}
For the conditional independence in Figure \ref{fig:non_star}(a), we know that:
\begin{equation}
\begin{aligned}
    \ind{\hat{q}}{4}{2}{3} &= \ind{\hat{q}}{4}{k}{2} \text{ By Lemma \ref{lemma:star_equality} on $\{X_2, X_3, X_4, X_k\}$},\\
    \ind{\hat{q}}{4}{1}{3} &= \ind{\hat{q}}{4}{k}{1} \text{ By Lemma \ref{lemma:star_equality} on $\{X_1, X_3, X_4, X_k\}$},\\
    \ind{\hat{q}}{4}{k}{2} &= \ind{\hat{q}}{4}{k}{1} \neq \ind{\hat{q}}{4}{1}{2}\text{ By Lemma \ref{lemma:equality_inequality}(a) and Lemma \ref{lemma:equality_inequality}(b)}\\
    &\text{on } \{X_1, X_2, X_k, X_4\}.
\end{aligned}
\end{equation}
This gives us $\ind{\hat{q}}{4}{2}{3} = \ind{\hat{q}}{4}{1}{3} \neq \ind{\hat{q}}{4}{1}{2}$. Similarly, we have $\ind{\hat{q}}{3}{2}{4} = \ind{\hat{q}}{3}{1}{4} \neq \ind{\hat{q}}{3}{1}{2}$.

We also know that:
\begin{equation}
\begin{aligned}
    \ind{\hat{q}}{2}{1}{3} &= \ind{\hat{q}}{2}{1}{k}\neq \ind{\hat{q}}{2}{k}{3} \text{ By Lemma \ref{lemma:equality_inequality}(a) and Lemma \ref{lemma:equality_inequality}(b)}\\
    &\text{on } \{X_1, X_2, X_k, X_3\},\\
    \ind{\hat{q}}{2}{1}{4} &= \ind{\hat{q}}{2}{1}{k}\neq \ind{\hat{q}}{2}{k}{4} \text{ By Lemma \ref{lemma:equality_inequality}(a) and Lemma \ref{lemma:equality_inequality}(b)}\\
    &\text{on } \{X_1, X_2, X_k, X_4\},\\
    \ind{\hat{q}}{2}{k}{3} &= \ind{\hat{q}}{2}{3}{4} =  \ind{\hat{q}}{2}{k}{4} \text{ By Lemma \ref{lemma:star_equality} on $\{X_2, X_3, X_4, X_k\}$},\\
    \ind{\hat{q}}{1}{2}{3} &= \ind{\hat{q}}{1}{2}{k}\neq \ind{\hat{q}}{1}{k}{3} \text{ By Lemma \ref{lemma:equality_inequality}(a) and Lemma \ref{lemma:equality_inequality}(b)}\\
    &\text{on } \{X_1, X_2, X_k, X_3\},\\
    \ind{\hat{q}}{1}{2}{4} &= \ind{\hat{q}}{1}{2}{k}\neq \ind{\hat{q}}{1}{k}{4} \text{ By Lemma \ref{lemma:equality_inequality}(a) and Lemma \ref{lemma:equality_inequality}(b)}\\
    &\text{on } \{X_1, X_2, X_k, X_4\},\\
    \ind{\hat{q}}{1}{k}{3} &= \ind{\hat{q}}{1}{3}{4} =  \ind{\hat{q}}{1}{k}{4} \text{ By Lemma \ref{lemma:star_equality} on $\{X_1, X_3, X_4, X_k\}$}.
\end{aligned}
\end{equation}
These equations imply $\ind{\hat{q}}{2}{1}{3} = \ind{\hat{q}}{2}{1}{4} \neq \ind{\hat{q}}{2}{3}{4}$ and $\ind{\hat{q}}{1}{2}{3} = \ind{\hat{q}}{1}{2}{4} \neq \ind{\hat{q}}{1}{3}{4}$.
If the conditional independence is as shown in Figure \ref{fig:non_star}(b), we have:
\begin{equation}
    \begin{aligned}
    &\ind{\hat{q}}{1}{2}{3} = \ind{\hat{q}}{1}{k'}{2} = \ind{\hat{q}}{1}{2}{4} = \ind{\hat{q}}{1}{k'}{4} \text{By Lemma \ref{lemma:star_equality} on }\\
    &\text{$\{X_1, X_2, X_3, X_{k'}\}$ and on $\{X_1, X_2, X_4, X_{k'}\}$},\\
    &\ind{\hat{q}}{1}{k'}{4} \neq \ind{\hat{q}}{1}{k}{4}\text{By Lemma \ref{lemma:equality_inequality}(b) on $\{X_1, X_k, X_{k'}, X_4\}$},\\
    &\ind{\hat{q}}{1}{k}{4} = \ind{\hat{q}}{1}{3}{4}\text{By Lemma \ref{lemma:star_equality} on $\{X_1, X_k, X_3, X_4\}$}.
    \end{aligned}
\end{equation}
These equations give us $\ind{\hat{q}}{1}{2}{3} = \ind{\hat{q}}{1}{2}{4} \neq \ind{\hat{q}}{1}{3}{4}$. Furthermore, by Equation (\ref{eq:ratio}), we have:
\begin{equation}\label{eq:non_star_ratios}
\begin{aligned}
    &\frac{\corr{1}{3}'\corr{2}{4}'}{\corr{1}{2}'\corr{3}{4}'} \leq \rho_{max}^2<1\\
    &\frac{\corr{1}{3}'\corr{2}{4}'}{\corr{1}{4}'\corr{2}{3}'} = 1
    \end{aligned}
\end{equation}
By symmetry, the remaining conditions in Equation (\ref{eq:non_star_cond}) are also satisfied.

\begin{figure}
    \centering
    \includegraphics[scale = 0.35]{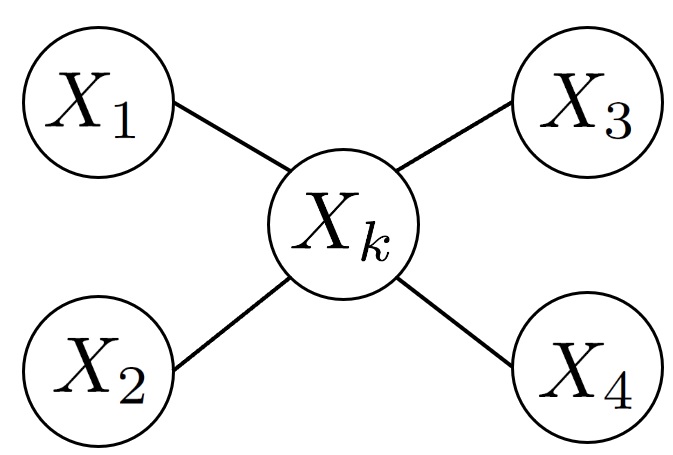}
    \caption{Possible conditional independence relations for a star shape.}
    \label{fig:star2}
\end{figure}

When the 4 nodes form a star structure in the tree, their conditional independence is given by either Figure \ref{fig:star} or there exists a node $X_k$ such that the conditional independence is as shown in Figure \ref{fig:star2}. Lemma \ref{lemma:star_equality} proves that Equation \ref{eq:star_cond} is satisfied if the conditional independence is given by Figure \ref{fig:star}. If the conditional independence is given by Figure \ref{fig:star2}, we have:
\begin{equation}
    \begin{aligned}
    \ind{\hat{q}}{1}{2}{3} = \ind{\hat{q}}{1}{2}{k} = \ind{\hat{q}}{1}{k}{3}\text{ By Lemma \ref{lemma:star_equality} on $\{X_1, X_2, X_3, X_k\}$},\\
    \ind{\hat{q}}{1}{4}{3} = \ind{\hat{q}}{1}{4}{k} = \ind{\hat{q}}{1}{k}{3}\text{ By Lemma \ref{lemma:star_equality} on $\{X_1, X_3, X_4, X_k\}$},\\
    \ind{\hat{q}}{1}{2}{4} = \ind{\hat{q}}{1}{2}{k} = \ind{\hat{q}}{1}{k}{4}\text{ By Lemma \ref{lemma:star_equality} on $\{X_1, X_2, X_4, X_k\}$}.
    \end{aligned}
\end{equation}
This implies that $\ind{\hat{q}}{1}{2}{3} = \ind{\hat{q}}{1}{4}{3} = \ind{\hat{q}}{1}{4}{2}$. By symmetry, all the remaining conditions of Equation \ref{eq:star_cond} are also satisfied.

This completes the proof that just by having access to the noisy probability distribution, it is possible to categorize any set of 4 nodes as a star/non-star shape.

\section{Proof of Theorem \ref{thm:unident}}\label{app:thm2_proof}
Given the noisy variance $\cov{i}{j}'$ and an estimate of the error probability vector $\mathbf{\hat{q}}$, we estimate the non-noisy covariance as:
\begin{equation}\label{eq:corr_estimate}
    \hat{\Sigma}_{i,j} = \dfrac{\cov{i}{j}'}{(1-2\hat{q}_i)(1-2\hat{q}_j)} = \dfrac{\cov{i}{j}(1-2q_i)(1-2q_j)}{(1-2\hat{q}_i)(1-2\hat{q}_j)}\forall i\neq j.
\end{equation}
% From Equation (\ref{eq:noisy_mean}), we know that:
% \begin{equation*}
%     \Sigma'_{i,i} = 1 - (1-2q_i)^2\E{X_i}^2,
% \end{equation*}
% which gives us 
% \begin{equation*}
% \E{X_i}^2 = \frac{1 - \Sigma'_{i,i}}{(1-2q_i)^2}.
% \end{equation*}
% Therefore the variance $\Sigma_{i,i}$ is given by:
% \begin{equation}
%     \begin{aligned}
%     \Sigma_{i,i} &= 1 - \E{X_i}^2\\
%     &= 1 - \frac{1 - \Sigma'_{i,i}}{(1-2q_i)^2}.
%     \end{aligned}
% \end{equation}
For the error probability vector $\mathbf{\hat{q}}$, using Equation (\ref{eq:noisy_var}) the non-noisy variance is estimated as:
\begin{equation}
\hat{\Sigma}_{i,i} = 1 - \frac{1 - \cov{i}{i}'}{(1 - 2\hat{q}_i)^2}
\end{equation}

To check if any conditional independence relation $X_i\perp X_j|X_k$ is true, we need to verify if it satisfies the correlation decay equation $\hat{\Sigma}_{i,j}\hat{\Sigma}_{k,k} = \hat{\Sigma}_{i,k}\hat{\Sigma}_{k,j}$. 

We first consider $T'$ where only one leaf node exchanges position with its neighbor. Suppose in the original tree, node $X_1$ is a leaf node connected to node $X_2$.

Consider the error vector $\mathbf{\hat{q}}$:
\begin{equation}\label{eq:q_hat}
    \begin{aligned}
        \hat{q}_i &= q_i \text{ $\forall$ } i\neq 1,2,\\
        \hat{q}_1 &= \frac{1}{2}\left(1-(1-2q_1)\sqrt{\frac{\cov{1}{2}^2}{\cov{2}{2}} - \cov{1}{1} + 1}\right),\\
        \hat{q}_2 &= 0.
    \end{aligned}
\end{equation}
To prove that this error vector results in $T'$, we need to prove that any node $X_k \neq X_1, X_2$ which satisfies $X_1\perp X_k|X_2$ in $T^*$ must satisfy $X_2\perp X_k|X_1$ in $T'$. We note that:
\begin{equation}
\begin{aligned}
\hat{\Sigma}_{1,2} &= \frac{\cov{1}{2}(1-2q_1)(1-2q_2)}{(1 - 2\hat{q}_1)}, \hat{\Sigma}_{1,k} = \frac{\cov{1}{k}(1-2q_1)}{(1-2\hat{q}_1)}\\
\hat{\Sigma}_{2,k} &= \cov{2}{k}(1-2q_2), \hat{\Sigma}_{1,1} = 1 - \frac{(1 - \cov{1}{1})(1-2q_1)^2}{(1 - 2\hat{q}_1)^2}
\end{aligned}
\end{equation}
Using $\cov{1}{k}\cov{2}{2} = \cov{1}{2}\cov{2}{k}$, it is easy to check that $\hat{\Sigma}_{2,k}\hat{\Sigma}_{1,1} = \hat{\Sigma}_{1,k} \hat{\Sigma}_{1,2}$.

Furthermore, we need to prove that any pair of nodes $X_{k_1},X_{k_2}\neq X_1, X_2$ such that $X_{k_1}\perp X_{k_2}|X_2$ in $T^*$ satisfy $X_{k_1}\perp X_{k_2}|X_1$ in $T'$. Doing similar substitutions by replacing node 2 by node $k_1$ and node $k$ by node $k_2$ gives us $\hat{\Sigma}_{1,1}\hat{\Sigma}_{k_1, k_2} = \hat{\Sigma}_{1,k_1}\hat{\Sigma}_{1, k_2}$ which proves that $X_{k_1}\perp X_{k_2}|X_1$.
%\begin{equation*}
%    \hat{\rho}_{k_1,k_2} = \corr{k_1}{2}\corr{2}{k_2}, \hat{\rho}_{1,k_1} = \corr{2}{k_1}, \hat{\rho}_{1,k_2} = \corr{2}{k_2},
%\end{equation*}

% in $T'$ as $\hat{\rho}_{k_1,k_2} = \hat{\rho}_{1,k_1}\hat{\rho}_{1,k_2}$

The remaining conditional independences not involving $X_1$ and $X_2$ remain intact as the error probability for the remaining nodes is assigned to the original probability of error.

Now, consider a tree $T'$ in which a set of leaf nodes $\mathcal{S}'$ exchange positions with their neighbors. For this case consider the error probability vector $\mathbf{\hat{q}}$ such that:
\begin{equation*}
    \begin{aligned}
        \hat{q}_i =& \frac{1}{2}\left(1-(1-2q_i)\sqrt{\frac{\cov{i}{j}^2}{\cov{j}{j}} - \cov{i}{i} + 1}\right), \\
        			  &\text{ $\forall$ } i\in \mathcal{S}', j = Parent(i)\\
        \hat{q}_j =& 0 \text{ $\forall$ } i\in \mathcal{S}', j = Parent(i)\\
        \hat{q}_k =& q_k, \text{ otherwise.}
    \end{aligned}
\end{equation*}

This is obtained by performing the same procedure on each leaf node one by one.

\section{Pseudo-code, Proof of Correctness, Time-complexity}\label{app:algo_details}
\subsection{Pseudo-code proof of correctness and Time-complexity for \fa}
The pseudo-code is presented in Algorithm \ref{alg:FindEC}.
\begin{algorithm}
\caption{Algorithm to find the Equivalence Cluster in $\mathcal{X}_{sub}\cup X_i$ containing node $X_i$}\label{alg:FindEC}
\begin{algorithmic}[1]
\Procedure{FindEC}{$X_i$, $\mathcal{X}_{sub}$}
% \If{$|\setx|$ == 0}
% \State{return [ ]}
% \EndIf
\If{$|\setx_{sub}|==0$}
\State{return [$X_i$  ]}
\EndIf
\If{$|\setx_{sub}|==1$}
\State{$\sete.append((X_i, \setx_{sub}[0]))$}
\State{return $X_i \cup\setx_{sub}$}
\EndIf
\If{$|\setx_{sub}|==2$}
\State{$\sete.append((X_i, \setx_{sub}[0]))$}
\State{$\sete.append((\setx_{sub}[1], \setx_{sub}[0]))$}
\State{return $X_i \cup \setx_{sub}$}
\EndIf
\State{$EC\gets $[ ]}
\State{$candidate\_nodes \gets \tilde{\mathcal{P}}_i^1 \cap \mathcal{X}$}
\For{$X_j$ in $candidate\_nodes$} \
% \State{$test\_nodes \gets $}
\State{$X_{k_1} \in candidate\_nodes\setminus X_j$}
\State{$IsInEC \gets True$}
\State{$test\_nodes\gets \tilde{\mathcal{P}}^2_i \cap \tilde{\mathcal{P}}^2_j \cap \tilde{\mathcal{P}}_{k_1}^2 \cap\setx_{sub}$}
\For{$X_{k_2}$ in $test\_nodes$}
\If{$min(\frac{\corr{i}{k_1}'\corr{j}{k_2}'}{\corr{i}{k_2}'\corr{j}{k_1}'}, \frac{\corr{i}{k_2}'\corr{j}{k_1}'}{\corr{i}{k_1}'\corr{j}{k_2}'} )<t$}
\State{$IsInEC\gets False$}
\EndIf
\EndFor
\If{$IsInEC == True$}
\State{$EC.$append$(X_{k_1})$}
\EndIf
\EndFor
\State{$\sete.append(X_i, EC[0])$}
\If{$len(EC)>1$}
\For{$i = 1$ to $len(EC)$}
\State{$\sete.append(EC[i], EC[0])$}
\EndFor
\EndIf
\State{return $X_i \cup EC$}
\EndProcedure
\end{algorithmic}
\end{algorithm}
\subsubsection*{Proof of correctness}
By the definition of $\tilde{\setp{1}}$, $candidate\_nodes$ contains at least all the nodes within a radius of 4 of $X_i$. Therefore, it contains all the nodes in the equivalence cluster containing $X_i$. Now we check for each node in $candidate\_nodes$ whether it is in the equivalence cluster of $X_i$. 

For any node $X_j$ to be in the equivalence cluster of $X_i$, every set of 4 nodes $\{X_i, X_j, X_{k_1}, X_{k_2}\}$ need to either be a star shape or a non-star shape where $(X_i, X_j)$ form a pair. 

In order to check this it is sufficient to check the condition for any set of 4 nodes $\{X_i, X_j, X_{k_1}, X_{k_2}\}$, for any random $X_{k_1}\in candidate\_nodes\setminus X_j$ and every $X_{k_2}\in \tilde{\mathcal{P}}^2_i \cap \tilde{\mathcal{P}}^2_j \cap \tilde{\mathcal{P}}_{k_1}^2\cap\setx_{sub}$. Suppose $X_j$ is not in the same equivalence cluster as $X_i$. This implies that there exist nodes in the path from $X_i$ to $X_j$. By the definition of the set $\tilde{\mathcal{P}}^2$, the node with a potential edge with $X_i$ is in $\tilde{\mathcal{P}}^2_i \cap \tilde{\mathcal{P}}^2_j \cap \tilde{\mathcal{P}}_{k_1}^2 \cap\setx_{sub}$. This will result in $\{X_i, X_j, X_{k_1}, X_{k_2}\}$ forming a non-star such that $(X_i, X_{k_1})$ form a pair.

\subsubsection*{Time-complexity}
As there are 2 for loops, it is easy to see that this subroutine is $\Or(n^2)$.

\subsection{Pseudo-code, proof of correctness and Time-complexity for \fb}

\begin{algorithm}
\caption{Algorithm to split nodes close to $X_i$ in $\setx\setminus \setx_{last}$ which have an edge with $EC(X_i)$. $\setx_{last}$ is such that $EC(X_i)\cup X_{leaf}\in \setx_{last}$.}\label{alg:splittree}
\begin{algorithmic}[1]
\Procedure{SplitTree}{$X_i, X_{leaf}, \mathcal{X}_{last}$}

\State{$close\_nodes \gets \tilde{\mathcal{P}}_i^1 \cap\tilde{\mathcal{P}}_{ext}^1\cap \bar{\mathcal{X}}_{last}$}
\State{$split \gets |close\_nodes|\times |close\_nodes|$ array of 0}
\For{$X_p$ in $close\_nodes$}
\For{$X_q\neq X_p$ in $\tilde{\mathcal{P}}_i^1 \cap\tilde{\mathcal{P}}_{ext}^1 \cap \tilde{\mathcal{P}}^2_p$}
\If{$\{X_i, X_{leaf}, X_p, X_q\}$ is non-star and $X_i, X_{leaf}$ form a pair}
\If{$X_q$ is in $\mathcal{X}_{last}$}
\State{$split_{p, :} \gets 0$ }
\State{Break from the inner loop}
\Else
\State{$split_{p,q} \gets 1$}
\EndIf
\EndIf
\EndFor
\EndFor
\State{Merge rows of $split$ with 1 in the same column}
\State{Return unique rows of $split$}
\EndProcedure
\end{algorithmic}
\end{algorithm}
The pseudo code is provided in Algorithm \ref{alg:splittree}.
\subsubsection*{Proof of Correctness}
$\setx_{last}$ is such that $EC(X_i), X_{leaf}\in \setx_{last}$. Furthermore, if any node $X_k\in \setx_{last}$ then the nodes from the equivalence cluster that has an edge with $EC(X_{leaf})$ from that connected component obtained by removing edges with $EC(X_i)$ are also in $\setx_{last}$. 
Let $\setx_{cc}$ denote the set of nodes in the subtrees of the original tree obtained by removing $X_i$ that don't contain any node from $\setx_{last}$.
We split the nodes in $\setx_{cc}\cap \tilde{\setp{1}}\cap\tilde{\mathcal{P}}_{ext}^1$ into subtrees. We first check if a node $X_k\in \setx_{cc}$. $X_k\in \tilde{\setp{1}}\cap\tilde{\mathcal{P}}_{ext}^1$ and $X_k\not\in \setx_{cc}$ if and only if $X_k \in \setx_{last}$ or there exists $X_{k'} \in \setx_{last}\cap \tilde{\mathcal{P}}_k^2$ such that $\{X_i, X_{leaf}, X_{k}, X_{k'}\}$ form a non-star and ($X_{k}$, $X_{k'}$) form a pair. This would happen when $X_{k'}$ is a node from the equivalence cluster that has an edge with $EC(X_{leaf})$. It is easy to see that $\{X_i, X_{leaf}, X_{k}\}$ are in each other's proximal sets. By the definition of $\tilde{\mathcal{P}}^2$, $X_{k'}\in\tilde{\mathcal{P}}^2_k$. Also $X_{k'}$ is within a radius of 4 from $X_i, X_{leaf}$, therefore $X_{k'}\in\tilde{\mathcal{P}}^2_i\cap\tilde{\mathcal{P}}^2_{ext}$.

% Any node $X_k$ in $\setx\setminus\setx_{last}$ such that $\{X_i, X_{leaf}, X_k, X_{k'}\}$ form a non-star such that $X_k, X_{k'}$ form a pair for some $X_{k'}\in \setx_{last}$ implies that $X_k$  belongs to a connected component which has at least one node from $\setx_{last}$. This follows from the definition of non-star shape and connected component.

For the nodes in $\setx_{cc}$, 2 nodes $X_{k_1}, X_{k_2}$ lie in the same subtree if $\{X_i, X_{leaf}, X_{k_1}, X_{k_2}\}$ form a non-star such that $(X_{k_1}, X_{k_2})$ are in the same pair. This is true from the definition of non-star shape. To implement this, we construct a square matrix $split$ of size $|\tilde{\mathcal{P}}_i^1 \cap\tilde{\mathcal{P}}_{ext}^1\cap \bar{\mathcal{X}}_{last}|\times |\tilde{\mathcal{P}}_i^1 \cap\tilde{\mathcal{P}}_{ext}^1\cap \bar{\mathcal{X}}_{last}|$ which contains 1 in position $(j_1, j_2)$ if $j_1^{th}$ and $j_2^{th}$ node in $\tilde{\mathcal{P}}_i^1 \cap\tilde{\mathcal{P}}_{ext}^1\cap \bar{\mathcal{X}}_{last}$ are $X_{k_1}$ and $X_{k_2}$ respectively and $\{X_i, X_{leaf}, X_{k_1}, X_{k_2}\}$ are in each others $\tilde{\mathcal{P}}^2$ and the set is classified as a non-star such that $(X_i, X_{leaf})$ form a pair. It has 0 at other positions. If it turns out that a node is not in $\setx_{cc}$, the row corresponding to that node is set to 0. Then, we merge the nodes which have 1 in common columns till no 2 rows have a 1 in a common column. 
The creation of the $split$ matrix and the merging are both $\mathcal{O}(n^2)$ operation. Therefore, the whole function is an $\mathcal{O}(n^2)$.

\subsection{Proof of Correctness of \fc }
\begin{lemma}\label{lemma:rec_proof}
$\fc$ correctly recovers all the edges between the nodes in $\mathcal{X}_{cc}$ as well as the edges between $\mathcal{X}_{cc}$ and $X_i$.
\end{lemma} 
\begin{proof}
We prove this lemma by induction on the number of nodes in $\setx_{cc}$.

\textbf{Base Case ($|\setx_{cc}| = 2$):}

Note that $|\setx_{cc}| = 1$ is not possible because if that were the case, $\setx_{cc}$ would be in $EC(X_i)$ and hence in $\setx_{last}$ which would imply $|\setx_{cc}| = 0$ which is a contradiction.

Also, note that when $|\setx_{cc}| = 2$, the nodes in $\setx_{cc}$ have to form a non-star when considered with $X_i$ and $X_{leaf}$ otherwise its nodes are going to be in $EC(X_i)$. By the correctness of $\fb$, $\setx_{cc}$ is correctly recognized as the only subtree in step 2. By the correctness of $\fa$, the algorithm correctly identifies that the nodes in $\setx_{cc}$ form an equivalence cluster and that this equivalence cluster has an edge with $EC(X_i)$. $\setx_{last}'$ now contains $\setx_{cc}$. The next recursion call terminates after step 2 as $subtrees$ is empty.

\textbf{Inductive Step: }

Suppose the Lemma \ref{lemma:rec_proof} is correct for $|\setx_{cc}|\leq k$.

Lets look at what happens when $|\setx_{cc}| = k+1$. 
\begin{itemize}
    \item In step 2, the algorithm correctly splits the nodes in $\setx_{cc}\cap \tilde{\setp{1}}\cap\tilde{\mathcal{P}}_{ext}^1$ into $subtrees$. Suppose this results in $j$ subtrees. Each of these subtrees have at least 2 nodes by the same argument as for the base case.
    \item By the correctness of $\fa$, step 4 correctly finds the equivalence cluster $EC$ in each subtree which has an edge with $EC(X_i)$. This proves the part of the lemma that the edges between $\setx_{cc}$ and $EC(X_i)$ are correctly identified.
    
    \item For each $subtree$ $\setx_{last}' =  \setx_{last}\cup (subtrees\setminus subtree)\cup EC$. Therefore $\setx_{last}'\supset \setx_{last}$.
    
    \item For the next recursion step, $X_i^+ = EC[0]$ and $\setx_{cc}^+\subset \setx_{cc}$ as $EC\in \setx_{cc}$ whereas $EC\not\in \setx_{cc}$.
    
    \item By the inductive assumption, all the edges within $\setx_{cc}^+$ are correctly recovered. Furthermore, the edges between $EC$ and $\setx_{cc}^+$ are correctly recovered. 
    
    \item This is true for all $subtree\in subtrees$ obtained in step 2.
    
    \item Therefore, all the edges within $\setx_{cc}$ are correctly recovered.
\end{itemize}

\end{proof}
\subsection{Proof of correctness and time complexity of the complete algorithm}

If the tree has just one Equivalence cluster, the algorithm discovers it by the correctness of $\fa$ and terminates. If the tree has more than one equivalence clusters, there would be at least one equivalence cluster with more than one nodes. By the correctness of $\fa$, the algorithm correctly finds one such equivalence cluster $EC$. It updates the $\setx_{last}$ with this equivalence cluster. Then it calls the $\fc$ function. The set $\setx_{cc}$ for this call to the $\fc$ function is $\setx\setminus EC$. By the correctness of $\fc$, the algorithm correctly recovers all the edges within $\setx\setminus EC$ and the edges between $\setx\setminus EC$ and $EC$. Therefore, it correctly recovers the equivalence class of the tree.

\subsection*{Time-complexity}
In the initialization step, $\fa$ is called at most $n$ times and hence takes $\Or(n^3)$ time. Within the $\fc$ function, for each call to $\fb$, $\fa$ is called at least once unless the $subtrees$ is empty, in which case $\fc$ terminates. Therefore, number of times $\fb$ is called within $\fc$ is upper bounded by the number of times $\fa$ is called within $\fc$+1. Whenever $\fa$ is called from within the $\fc$ function, it discovers at least one edge. Therefore $\fa$ can be called at most $n-2$ times within $\fc$. Therefore, the total time complexity of $\fc$ is $\Or(n^3)$. Therefore, the total time complexity of $\textsc{FindTree}$ is $\Or(n^3)$.

\section{Sample Complexity Analysis}\label{app:sample_complexity}

We first calculate how accurate we need $\tilde{\rho}_{i,j}$ to be in order to make sure that any set of $4$ nodes in each other's $\mathcal{P}^2$ are correctly classified as a star or a non-star. We then use the Hoeffding's concentration bound to get a high probability bound for sample complexity.

Consider a set of 4 nodes $\{X_{i_1}, X_{i_2}, X_{i_3}, X_{i_4}\}$ which forms non-star such that $(X_{i_1}, X_{i_2})$ form a pair. For these nodes to be correctly classified, we need:
\begin{equation*}
\begin{aligned}
    \frac{\tilde{\rho}_{i_1, i_3}\tilde{\rho}_{i_2, i_4}}{\tilde{\rho}_{i_1, i_4}\tilde{\rho}_{i_2, i_3}} > t_3,\\
    \frac{\tilde{\rho}_{i_1, i_3}\tilde{\rho}_{i_2, i_4}}{\tilde{\rho}_{i_1, i_2}\tilde{\rho}_{i_3, i_4}} < t_3,
\end{aligned}
\end{equation*}
where $t_3 = (1+\rho_{max}^2)/2$ and $\tilde{\rho}_{i,j}$ is the empirical correlation obtained from noisy samples between nodes $X_i$ and $X_j$. We look at the first condition. This is equivalent to the condition:
\begin{equation*}
\begin{aligned}
    \frac{\tilde{\Sigma}_{i_1, i_3}\tilde{\Sigma}_{i_2, i_4}}{\tilde{\Sigma}_{i_1, i_4}\tilde{\Sigma}_{i_2, i_3}} > t_3,\\
    \frac{\tilde{\Sigma}_{i_1, i_3}\tilde{\Sigma}_{i_2, i_4}}{\tilde{\Sigma}_{i_1, i_2}\tilde{\Sigma}_{i_3, i_4}} < t_3.
\end{aligned}
\end{equation*}

This would be true if:
\begin{equation}\label{eq:emp_corr}
    \left|\frac{\tilde{\Sigma}_{i_1, i_3}\tilde{\Sigma}_{i_2, i_4}}{\tilde{\Sigma}_{i_1, i_4}\tilde{\Sigma}_{i_2, i_3}} - \frac{{\Sigma}_{i_1, i_3}'{\Sigma}_{i_2, i_4}'}{{\Sigma}_{i_1, i_4}'{\Sigma}_{i_2, i_3}'}\right|< \frac{1-t_3}{2}
\end{equation}
as $\frac{{\Sigma}_{i_1, i_3}'{\Sigma}_{i_2, i_4}'}{{\Sigma}_{i_1, i_4}'{\Sigma}_{i_2, i_3}'} = 1$. Let ${\Sigma}_{i, j}' = \tilde{\Sigma}_{i, j} + \delta_{i, j}$. Then, we have:
\begin{align*}
    \frac{{\Sigma}_{i_1, i_3}'{\Sigma}_{i_2, i_4}'}{{\Sigma}_{i_1, i_4}'{\Sigma}_{i_2, i_3}'} &= \frac{(\tilde{\Sigma}_{i_1, i_3} + \delta_{i_1, i_3})(\tilde{\Sigma}_{i_2, i_4} + \delta_{i_2, i_4})}{(\tilde{\Sigma}_{i_1, i_4} + \delta_{i_1, i_4})(\tilde{\Sigma}_{i_2, i_3} + \delta_{i_2, i_3})}\\
    &= \frac{\tilde{\Sigma}_{i_1, i_3}\tilde{\Sigma}_{i_2, i_4}}{\tilde{\Sigma}_{i_1, i_4}\tilde{\Sigma}_{i_2, i_3}}
    \dfrac{\left( 1+ \frac{\delta_{i_1, i_3}}{\tilde{\Sigma}_{i_1, i_3}}\right)
    \left(1+ \frac{\delta_{i_2, i_4}}{\tilde{\Sigma}_{i_2, i_4}}\right)}
    { \left(1 + \frac{\delta_{i_1, i_4}}{\tilde{\Sigma}_{i_1, i_4}}\right) 
    \left(1 + \frac{\delta_{i_2, i_3}}{\tilde{\Sigma}_{i_2, i_3}}\right)}\\
    &= \frac{\tilde{\Sigma}_{i_1, i_3}\tilde{\Sigma}_{i_2, i_4}}{\tilde{\Sigma}_{i_1, i_4}\tilde{\Sigma}_{i_2, i_3}}
    \dfrac{\left( 1+ \frac{\delta_{i_1, i_3}}{\tilde{\Sigma}_{i_1, i_3}} + \frac{\delta_{i_2, i_4}}{\tilde{\Sigma}_{i_2, i_4}} + \frac{\delta_{i_1, i_3}\delta_{i_2, i_4}}{\tilde{\Sigma}_{i_1, i_3}\tilde{\Sigma}_{i_2, i_4}} \right)}
    { \left(1 + \frac{\delta_{i_1, i_4}}{\tilde{\Sigma}_{i_1, i_4}} + \frac{\delta_{i_2, i_3}}{\tilde{\Sigma}_{i_2, i_3}} +  \frac{\delta_{i_2, i_3}\delta_{i_1, i_4}}{\tilde{\Sigma}_{i_2, i_3}\tilde{\Sigma}_{i_1, i_4}}\right)} \\
   &\hspace{-1cm} \text{Setting } \delta = max_{i_1,i_2} |\delta_{i_1i_2}|, \tilde{\Sigma}_m = \min_{i,j} |\tilde{\Sigma}_{i, j}|, \text{ and if } \frac{\delta}{\tilde{\Sigma}_m} \leq 0.4, \text{ we get } \dots\\
    &\leq \frac{\tilde{\Sigma}_{i_1, i_3}\tilde{\Sigma}_{i_2, i_4}}{\tilde{\Sigma}_{i_1, i_4}\tilde{\Sigma}_{i_2, i_3}}
    \dfrac{\left( 1+ \frac{2\delta}{\tilde{\Sigma}_m} + \frac{\delta^2}{\tilde{\Sigma}_m^2}  \right)}
    { \left(1 - \frac{2\delta}{\tilde{\Sigma}_m}\right)} \\
    &\leq \frac{\tilde{\Sigma}_{i_1, i_3}\tilde{\Sigma}_{i_2, i_4}}{\tilde{\Sigma}_{i_1, i_4}\tilde{\Sigma}_{i_2, i_3}}
    \dfrac{\left( 1+ \frac{2.5\delta}{\tilde{\Sigma}_m} \right)}
    { \left(1 - \frac{2\delta}{\tilde{\Sigma}_m}\right)} \\
    &= \frac{\tilde{\Sigma}_{i_1, i_3}\tilde{\Sigma}_{i_2, i_4}}{\tilde{\Sigma}_{i_1, i_4}\tilde{\Sigma}_{i_2, i_3}}
\left( 1+ \frac{2.5\delta}{\tilde{\Sigma}_m} \right)
\left(1 + \frac{2\delta}{\tilde{\Sigma}_m} + \left(\frac{2\delta}{\tilde{\Sigma}_m} \right)^2 \dots \right) \\
&<\frac{\tilde{\Sigma}_{i_1, i_3}\tilde{\Sigma}_{i_2, i_4}}{\tilde{\Sigma}_{i_1, i_4}\tilde{\Sigma}_{i_2, i_3}}
\left( 1+ \frac{2.5\delta}{\tilde{\Sigma}_m} \right)\left(1 + \frac{6\delta}{\tilde{\Sigma}_m}\right)\\
    &< \frac{\tilde{\Sigma}_{i_1, i_3}\tilde{\Sigma}_{i_2, i_4}}{\tilde{\Sigma}_{i_1, i_4}\tilde{\Sigma}_{i_2, i_3}}
    \left( 1+ \frac{16\delta}{\tilde{\Sigma}_m } \right)
\end{align*}

Similarly, we can show the other side of inequality. 
\begin{align*}
    \frac{{\Sigma}_{i_1, i_3}'{\Sigma}_{i_2, i_4}'}{{\Sigma}_{i_1, i_4}'{\Sigma}_{i_2, i_3}'} \geq 
    \frac{\tilde{\Sigma}_{i_1, i_3}\tilde{\Sigma}_{i_2, i_4}}{\tilde{\Sigma}_{i_1, i_4}\tilde{\Sigma}_{i_2, i_3}}
\left( 1 - \frac{16\delta}{\tilde{\Sigma}_m} \right)\\
\end{align*}

Using $|\Sigma_{i,j}| \geq 0.5 t_2$ along with some basic algebraic manipulation yields:
\begin{align*}
    \left|\frac{\tilde{\Sigma}_{i_1, i_3}\tilde{\Sigma}_{i_2, i_4}}{\tilde{\Sigma}_{i_1, i_4}\tilde{\Sigma}_{i_2, i_3}} - \frac{{\Sigma}_{i_1, i_3}'{\Sigma}_{i_2, i_4}'}{{\Sigma}_{i_1, i_4}'{\Sigma}_{i_2, i_3}'}\right|&< \frac{\tilde{\Sigma}_{i_1, i_3}\tilde{\Sigma}_{i_2, i_4}}{\tilde{\Sigma}_{i_1, i_4}\tilde{\Sigma}_{i_2, i_3}}\cdot \frac{16\delta}{\tilde{\Sigma}_m} \\
    &<  \frac{16\delta}{\tilde{\Sigma}_m^3} \\
    &< \frac{128\delta}{t_2^3} 
\end{align*}
Comparing this with Equation (\ref{eq:emp_corr}), we get:
\begin{equation*}
    \delta<\frac{t_2^3(1-t_3)}{128}.
\end{equation*}
It is easy to check that the $\delta$ obtained for the other condition of non-star shape and for star shape is also the same.

Now we look at the concentration of empirical correlation random variable for the noisy samples from the tree structured Ising model.

Since $X_i'$, $X_j'$ and $X_i'X_j'$ have support on $\{-1,1\}$, we can use the Hoeffding's inequality to obtain the concentration bounds for them. Suppose we draw $m$ independent samples from the noisy distribution. Let $Z^i_1, Z^i_2, \dots Z^i_m$ denote the $m$ samples of $X'_i$, $Z^j_1, Z^j_2, \dots Z^j_m$ denote the $m$ samples of $X'_j$ and $Z^{ij}_1, Z^{ij}_2, \dots Z^{ij}_m$ denote the $m$ samples of $X'_j X'_i$. Now, if we have:
\begin{equation*}
\begin{aligned}
|\frac{1}{m}\sum_{k=1}^m Z^{ij}_k - \E{X_i'X_j'}| < \delta/2,\\
|\frac{1}{m^2}(\sum_{k=1}^m Z^{i}_k)(\sum_{k=1}^m Z^{j}_k) - \E{X_i'}\E{X_j'}| < \delta/2,
\end{aligned}
\end{equation*}
then $|\Sigma'_{i,j} - \tilde{\Sigma}_{i,j}| < \delta$.

Let $\bar{X_i'}$, $\bar{X_j'}$ and $\overline{X_i'X_j'}$ denote the sample means of $X_i'$, $X_j'$ and $X_i'X_j'$ respectively. Further, let $\bar{X_i'} = \E{X_i'} + \alpha_i$ and $\bar{X_j'} = \E{X_j'} + \alpha_j$ and $\alpha_{max} = \max\{|\alpha_i|, |\alpha_j|\}$
\begin{equation}
\begin{aligned}
    |\bar{X_i'}\bar{X_j'} - \E{X_i'}\E{X_j'}|&=|(\E{X_i'} + \alpha_i)(\E{X_j'} + \alpha_j)-\E{X_i'}\E{X_j'}|\\
    &= |\alpha_i\E{X_j'}+\alpha_j\E{X_i'}+\alpha_i\alpha_j|\\
    &\leq |\alpha_i||\E{X_j'}|+|\alpha_j||\E{X_i'}|+|\alpha_i||\alpha_j|\\
    &\leq 4\alpha_{max} \text{ (As $|\E{X_i'}|, |\E{X_j'}| < 1, |\alpha_i| < 2$)}
\end{aligned}
\end{equation}
Therefore, when $|\bar{X_i'} - \E{X_i'}|<\delta/8$ and $|\bar{X_j'} - \E{X_j'}|<\delta/8$, we have $|\bar{X_i'}\bar{X_j'} - \E{X_i'}\E{X_j'}|<\delta/2$

% \begin{equation*}
%   \begin{aligned}
%   |\frac{1}{m}\sum_{k=1}^m Z^{i}_k - \E{X_i'}| &< \frac{\delta}{4\mu_{max}},\\
%   |\frac{1}{m}\sum_{k=1}^m Z^{j}_k - \E{X_j'}| &< \frac{\delta}{4\mu_{max}}, \text{then}\\
%   |\frac{1}{m^2}(\sum_{k=1}^m Z^{i}_k)(\sum_{k=1}^m Z^{j}_k) - \E{X_i'}\E{X_j'}| &< \delta/2\text{ (ignoring high order $\delta$ terms).}
%   \end{aligned} 
% \end{equation*}
By Hoeffding's inequality we have:
% \begin{equation}
% \begin{aligned}
%     P(|\frac{1}{m}\sum_{k=1}^m Z^{ij}_k - \E{X_i'X_j'}|> \frac{\delta}{2}) &\leq 2\exp\left(\frac{-m (\delta/2)^2}{2}\right),\\
%     P(|\frac{1}{m}\sum_{k=1}^m Z^{i}_k - \E{X_i'}|> \frac{\delta}{4\mu_{max}}) &\leq 2\exp \left(\frac{-m(\delta/4\mu_{max})^2}{2}\right),\\
%     P(|\frac{1}{m}\sum_{k=1}^m Z^{j}_k - \E{X_j'}|> \frac{\delta}{4\mu_{max}}) &\leq 2\exp \left(\frac{-m(\delta/4\mu_{max})^2}{2}\right).
%     \end{aligned}
% \end{equation}
\begin{equation}
\begin{aligned}
    P(|\overline{X_i'X_j'} - \E{X_i'X_j'}|> \frac{\delta}{2}) &\leq 2\exp\left(\frac{-m (\delta/2)^2}{2}\right),\\
    P(|\bar{X_i'} - \E{X_i'}|> \frac{\delta}{8}) &\leq 2\exp \left(\frac{-m(\delta/8)^2}{2}\right),\\
    P(|\bar{X_j'} - \E{X_j'}|> \frac{\delta}{8}) &\leq 2\exp \left(\frac{-m(\delta/8)^2}{2}\right).
    \end{aligned}
\end{equation}
Therefore, 
\begin{equation}\label{eq:one_bad_event}
\begin{aligned}
P(|\Sigma'_{i,j} - \tilde{\Sigma}_{i,j}| > \delta) &\leq 2\exp\left(\frac{-m (\delta/2)^2}{2}\right)+ 4\exp\left(\frac{-m (\delta/8)^2}{2}\right)\\
&\leq 6\exp\left(\frac{-m (\delta/8)^2}{2}\right)
\end{aligned}
\end{equation}
We define a bad event $\mathcal{B}$ as follows:
\begin{align}
    \mathcal{B} := \left\{ |\tilde{\Sigma}_{i, j} - \Sigma_{i,j}'|> \delta \text{ for any } i, j \in [n], i \neq j\right\}
\end{align}
Using union bound on Equation \ref{eq:one_bad_event}, we bound the probability of  $\mathcal{B}$ by:
\begin{equation}
    P(\mathcal{B})\leq (n^2)\left( 6\exp\left(\frac{-m (\delta/8)^2}{2}\right)\right).
\end{equation}
Therefore, for the probability of mistake to be bounded by $\tau$, the number of samples is given by:
\begin{equation}
    m \geq \frac{128}{\delta^2}\log\left(\frac{6n^2}{\tau}\right)
\end{equation}

\bibliography{references}

\begin{thebibliography}{10}

\bibitem{anandkumar2012learning}
A.~Anandkumar, D.~J. Hsu, F.~Huang, and S.~M. Kakade.
\newblock Learning mixtures of tree graphical models.
\newblock In {\em Advances in Neural Information Processing Systems}, pages
  1052--1060, 2012.

\bibitem{bresler2015efficiently}
G.~Bresler.
\newblock Efficiently learning ising models on arbitrary graphs.
\newblock In {\em Proceedings of the forty-seventh annual ACM symposium on
  Theory of computing}, pages 771--782. ACM, 2015.

\bibitem{bresler2014hardness}
G.~Bresler, D.~Gamarnik, and D.~Shah.
\newblock Hardness of parameter estimation in graphical models.
\newblock In {\em Advances in Neural Information Processing Systems}, pages
  1062--1070, 2014.

\bibitem{bresler2014structure}
G.~Bresler, D.~Gamarnik, and D.~Shah.
\newblock Structure learning of antiferromagnetic ising models.
\newblock In {\em Advances in Neural Information Processing Systems}, pages
  2852--2860, 2014.

\bibitem{bresler2016learning}
G.~Bresler and M.~Karzand.
\newblock Learning a tree-structured ising model in order to make predictions.
\newblock {\em arXiv preprint arXiv:1604.06749}, 2016.

\bibitem{bresler2008reconstruction}
G.~Bresler, E.~Mossel, and A.~Sly.
\newblock Reconstruction of markov random fields from samples: Some
  observations and algorithms.
\newblock In {\em Approximation, Randomization and Combinatorial Optimization.
  Algorithms and Techniques}, pages 343--356. Springer, 2008.

\bibitem{brush1967history}
S.~G. Brush.
\newblock History of the lenz-ising model.
\newblock {\em Reviews of modern physics}, 39(4):883, 1967.

\bibitem{chenlearning}
Y.~Chen.
\newblock Learning sparse ising models with missing data.

\bibitem{choi2010exploiting}
M.~J. Choi, J.~J. Lim, A.~Torralba, and A.~S. Willsky.
\newblock Exploiting hierarchical context on a large database of object
  categories.
\newblock In {\em 2010 IEEE Computer Society Conference on Computer Vision and
  Pattern Recognition}, pages 129--136. IEEE, 2010.

\bibitem{chow1968approximating}
C.~Chow and C.~Liu.
\newblock Approximating discrete probability distributions with dependence
  trees.
\newblock {\em IEEE transactions on Information Theory}, 14(3):462--467, 1968.

\bibitem{dasgupta1999learning}
S.~Dasgupta.
\newblock Learning polytrees.
\newblock In {\em Proceedings of the Fifteenth conference on Uncertainty in
  artificial intelligence}, pages 134--141. Morgan Kaufmann Publishers Inc.,
  1999.

\bibitem{daskalakis2019testing}
C.~Daskalakis, N.~Dikkala, and G.~Kamath.
\newblock Testing ising models.
\newblock {\em IEEE Transactions on Information Theory}, 2019.

\bibitem{goel2019learning}
S.~Goel, D.~M. Kane, and A.~R. Klivans.
\newblock Learning ising models with independent failures.
\newblock {\em arXiv preprint arXiv:1902.04728}, 2019.

\bibitem{hamilton2017information}
L.~Hamilton, F.~Koehler, and A.~Moitra.
\newblock Information theoretic properties of markov random fields, and their
  algorithmic applications.
\newblock In {\em Advances in Neural Information Processing Systems}, pages
  2463--2472, 2017.

\bibitem{ising1925beitrag}
E.~Ising.
\newblock Beitrag zur theorie des ferromagnetismus.
\newblock {\em Zeitschrift f{\"u}r Physik A Hadrons and Nuclei},
  31(1):253--258, 1925.

\bibitem{jaimovich2006towards}
A.~Jaimovich, G.~Elidan, H.~Margalit, and N.~Friedman.
\newblock Towards an integrated protein--protein interaction network: A
  relational markov network approach.
\newblock {\em Journal of Computational Biology}, 13(2):145--164, 2006.

\bibitem{pmlr-v97-katiyar19a}
A.~Katiyar, J.~Hoffmann, and C.~Caramanis.
\newblock Robust estimation of tree structured {G}aussian graphical models.
\newblock In {\em Proceedings of the 36th International Conference on Machine
  Learning}, volume~97 of {\em Proceedings of Machine Learning Research}, pages
  3292--3300. PMLR, 2019.

\bibitem{klivans2017learning}
A.~Klivans and R.~Meka.
\newblock Learning graphical models using multiplicative weights.
\newblock In {\em 2017 IEEE 58th Annual Symposium on Foundations of Computer
  Science (FOCS)}, pages 343--354. IEEE, 2017.

\bibitem{kolar2012estimating}
M.~Kolar and E.~P~Xing.
\newblock Estimating sparse precision matrices from data with missing values.
\newblock 2012.

\bibitem{lee2007efficient}
S.-I. Lee, V.~Ganapathi, and D.~Koller.
\newblock Efficient structure learning of markov networks using $ l\_1
  $-regularization.
\newblock In {\em Advances in neural Information processing systems}, pages
  817--824, 2007.

\bibitem{lindgrenrobust}
E.~M. Lindgren, V.~Shah, Y.~Shen, A.~G. Dimakis, and A.~Klivans.
\newblock On robust learning of ising models.

\bibitem{liu2012high}
H.~Liu, F.~Han, M.~Yuan, J.~Lafferty, L.~Wasserman, et~al.
\newblock High-dimensional semiparametric gaussian copula graphical models.
\newblock {\em The Annals of Statistics}, 40(4):2293--2326, 2012.

\bibitem{loh2011high}
P.-L. Loh and M.~J. Wainwright.
\newblock High-dimensional regression with noisy and missing data: Provable
  guarantees with non-convexity.
\newblock In {\em Advances in Neural Information Processing Systems}, pages
  2726--2734, 2011.

\bibitem{lokhov2018optimal}
A.~Y. Lokhov, M.~Vuffray, S.~Misra, and M.~Chertkov.
\newblock Optimal structure and parameter learning of ising models.
\newblock {\em Science advances}, 4(3):e1700791, 2018.

\bibitem{martinelli2003ising}
F.~Martinelli, A.~Sinclair, and D.~Weitz.
\newblock The ising model on trees: Boundary conditions and mixing time.
\newblock In {\em 44th Annual IEEE Symposium on Foundations of Computer
  Science, 2003. Proceedings.}, pages 628--639. IEEE, 2003.

\bibitem{mesnards2018detecting}
N.~G.~d. Mesnards and T.~Zaman.
\newblock Detecting influence campaigns in social networks using the ising
  model.
\newblock {\em arXiv preprint arXiv:1805.10244}, 2018.

\bibitem{pmlr-v89-nikolakakis19a}
K.~E. Nikolakakis, D.~S. Kalogerias, and A.~D. Sarwate.
\newblock Learning tree structures from noisy data.
\newblock In {\em Proceedings of Machine Learning Research}, volume~89 of {\em
  Proceedings of Machine Learning Research}, pages 1771--1782. PMLR, 2019.

\bibitem{nikolakakis2019non}
K.~E. Nikolakakis, D.~S. Kalogerias, and A.~D. Sarwate.
\newblock Non-parametric structure learning on hidden tree-shaped
  distributions.
\newblock {\em arXiv preprint arXiv:1909.09596}, 2019.

\bibitem{ravikumar2010high}
P.~Ravikumar, M.~J. Wainwright, J.~D. Lafferty, et~al.
\newblock High-dimensional ising model selection using l1-regularized logistic
  regression.
\newblock {\em The Annals of Statistics}, 38(3):1287--1319, 2010.

\bibitem{roth2005fields}
S.~Roth and M.~J. Black.
\newblock Fields of experts: A framework for learning image priors.
\newblock In {\em 2005 IEEE Computer Society Conference on Computer Vision and
  Pattern Recognition (CVPR'05)}, volume~2, pages 860--867. Citeseer, 2005.

\bibitem{schneidman2006weak}
E.~Schneidman, M.~J. Berry~II, R.~Segev, and W.~Bialek.
\newblock Weak pairwise correlations imply strongly correlated network states
  in a neural population.
\newblock {\em Nature}, 440(7087):1007, 2006.

\bibitem{srebro2003maximum}
N.~Srebro.
\newblock Maximum likelihood bounded tree-width markov networks.
\newblock {\em Artificial intelligence}, 143(1):123--138, 2003.

\bibitem{vuffray2016interaction}
M.~Vuffray, S.~Misra, A.~Lokhov, and M.~Chertkov.
\newblock Interaction screening: Efficient and sample-optimal learning of ising
  models.
\newblock In {\em Advances in Neural Information Processing Systems}, pages
  2595--2603, 2016.

\bibitem{wang2017robust}
L.~Wang and Q.~Gu.
\newblock Robust gaussian graphical model estimation with arbitrary corruption.
\newblock In {\em Proceedings of the 34th International Conference on Machine
  Learning-Volume 70}, pages 3617--3626. JMLR. org, 2017.

\bibitem{wu2018sparse}
S.~Wu, S.~Sanghavi, and A.~G. Dimakis.
\newblock Sparse logistic regression learns all discrete pairwise graphical
  models.
\newblock {\em arXiv preprint arXiv:1810.11905}, 2018.

\bibitem{yang2015robust}
E.~Yang and A.~C. Lozano.
\newblock Robust gaussian graphical modeling with the trimmed graphical lasso.
\newblock In {\em Advances in Neural Information Processing Systems}, pages
  2602--2610, 2015.

\bibitem{zhou2007self}
W.-X. Zhou and D.~Sornette.
\newblock Self-organizing ising model of financial markets.
\newblock {\em The European Physical Journal B}, 55(2):175--181, 2007.

\end{thebibliography}
\bibliographystyle{abbrv}

%\input{appendix.tex}
%\begin{acks}
%	
%	The authors would like to thank Taylor Kessler Faulkner and Surbhi Goel for bouncing ideas, and their pointers on making this paper more enjoyable to read.
%	
%	The work is partially supported by the \grantsponsor{GS501100001809}{National Science Foundation}{http://dx.doi.org/10.13039/501100001809} under Grants
%	No.:~\grantnum{GS501100001809}{1302435}, No.:~\grantnum{GS501100001809}{1609279} and
%	No.:~\grantnum{GS100000144}{1704778}.
%	
%	
%\end{acks}
\end{document}